\documentclass[journal]{IEEEtran}
\usepackage{booktabs}
\usepackage{amsmath}
\usepackage{mathtools}
\usepackage{makecell}
\usepackage{amsfonts}
\usepackage{threeparttable} 
\usepackage{graphicx}
\usepackage{algorithm}
\usepackage{algorithmic}
\usepackage{multirow}
\usepackage{subcaption}
\usepackage{pifont}
\usepackage{xcolor}
\usepackage{etoolbox}
\usepackage{amsthm}
\newtheorem{theorem}{Theorem}
\newtheorem{lemma}[theorem]{Lemma}

\ifCLASSINFOpdf
\else
\fi
\usepackage[style=numeric,sorting=none]{biblatex}
\usepackage{hyperref}

\begin{document}
%
\title{CAST: Continuous and Differentiable Semi-Structured Sparsity-Aware Training for Large Language Models}
%
%
%

\author{Weiyu Huang, Yuezhou Hu, Jun~Zhu,~\IEEEmembership{Fellow,~IEEE}         and~Jianfei~Chen,~\IEEEmembership{Member,~IEEE,}
\thanks{Weiyu Huang, Yuezhou Hu , Jun Zhu and Jianfei Chen are with the Department of Computer Science and Technology, Institute for AI, BNRist Center, THBI Lab, Tsinghua-Bosch Joint ML Center, Tsinghua University (e-mail: hwy23@mails.tsinghua.edu.cn; huyz21@mails.tsinghua.edu.cn; dcszj@tsinghua.edu.cn; jianfeic@tsinghua.edu.cn;)}
}

%
%

\markboth{Submitted to IEEE TRANSACTIONS ON PATTERN ANALYSIS AND MACHINE INTELLIGENCE on July 2025}%
{Shell \MakeLowercase{\textit{et al.}}: Bare Demo of IEEEtran.cls for IEEE Journals}
%



\maketitle

\begin{abstract}
 Sparsity-aware training is an effective approach for transforming large language models (LLMs) into hardware-friendly sparse patterns, thereby reducing latency and memory consumption during inference. In this paper, we propose Continuous Adaptive Sparse Trainer (CAST), a fully continuous and differentiable sparsity-aware training framework for semi-structured (or “N:M”) sparse models. Unlike previous approaches that optimize sparsity patterns and weights separately, CAST enables seamless joint optimization during training, while progressively transforming the model into the desired sparsity format. Specifically, CAST introduces three key components: 1) \textit{AdamS}, a sparsity-aware optimizer that leverages adaptive $L_1$ decay to promote uniform sparsification across all parameters; 2) \textit{Weight Scaling}, a module designed to mitigate the magnitude reduction caused by decay while preserving desired sparsity patterns; 3) \textit{Knowledge Distillation}, which employs the dense model as a self-teacher to enhance training efficiency. We evaluate CAST under 2:4 sparsity patterns across multiple model families, ranging from 125M to 13B parameters. Our results demonstrate significant improvements over previous state-of-the-art methods in both perplexity and zero-shot accuracy with minimal training resources. Notably, on LLaMA2-7B, our 2:4 sparse model achieves a negligible perplexity increase of 0.09 and a 0.36\% gain in zero-shot accuracy compared to the dense model using only 2\% of the original pretraining tokens. Additionally, we establish an accurate and robust empirical scaling law to predict sparse model performance given adequate training resources. Finally, we demonstrate the practical applicability of our sparse models by evaluating them under quantization and fine-tuning scenarios.
\end{abstract}


%
\IEEEpeerreviewmaketitle

\section{Introduction}
%
%
%
%




\IEEEPARstart{T}{he} growing scale of Large Language Models (LLMs) \cite{brown2020language, devlin2018bert, grattafiori2024llama} has led to increasing demands for memory and computation during deployment. This has sparked growing interest in model compression techniques designed to address these challenges. Among existing approaches, weight sparsity \cite{10.5555/3546258.3546499} has emerged as a particularly effective solution. By setting a proportion of model parameters to zero and thereby removing them from storage and computation, weight sparsity can significantly reduce memory footprints and accelerate inference in practical deployment scenarios.

Among various weight sparsity patterns, semi-structured (N:M) sparsity—where each group of M parameters retains only N nonzero values—has shown particular promise for balancing hardware efficiency and model accuracy \cite{ mishra2021acceleratingsparsedeepneural, NEURIPS2021_6e8404c3, DBLP:conf/iclr/ZhouMZLZYSL21}. Specifically, the 2:4 sparsity pattern is natively supported by NVIDIA GPUs, thereby enabling faster matrix multiplications and reduced memory latency. This support can yield up to a 2× speedup in both the prefill and decoding stages during inference. Moreover, the fine-grained nature of semi-structured sparsity allows for more precise control over weight pruning, offering greater potential to preserve accuracy compared to coarser-grained alternatives such as structured sparsity. 

Despite the promise of N:M sparsity, its application to large language models still remains largely underexplored. Most existing methods induce N:M sparsity through one-shot pruning \cite{frantar2023sparsegpt, sun2023simple}, which relies on handcrafted importance metrics to identify and remove unimportant weights. While computationally efficient, such methods often severely degrade a model's capacity for challenging zero-shot inference, as manually assigned metrics fail to capture the true importance of individual weights. For example, applying the classic one-shot pruning method Wanda \cite{sun2023simple} to LLaMA2-7B under the 2:4 sparsity pattern reduces its MMLU accuracy from 45.3\% to 27.6\%—approaching the level of random guessing and thus significantly limiting its potential for practical deployment.

To overcome the limitations of heuristic-driven pruning, sparsity-aware training (SAT) enables neural networks to learn and adapt to predefined sparsity patterns during training, achieving desired sparsity without compromising performance. However, prior SAT studies have primarily focused on vision and smaller-scale language models, typically relying on iterative \cite{DBLP:conf/iclr/FrankleC19} or manual pruning and regrowth strategies \cite{mishra2021acceleratingsparsedeepneural, han2015deep, NIPS2015_ae0eb3ee} to discover optimal sparsity patterns. While effective in their context, these methods do not transfer well to LLMs, where training costs are orders of magnitude higher. As a more efficient alternative, recent works such as Wanda \cite{sun2023simple} and MaskLLM \cite{NEURIPS2024_0e9a05f5} apply SAT to pretrained LLMs by either performing simple retraining after one-shot pruning or optimizing only the sparsity patterns (i.e., masks). Despite improved efficiency, these approaches overlook the intertwined relationship between model masks and weights. This lack of joint optimization limits their overall effectiveness and achievable performance. Therefore, in this paper, we aim to develop an \textit{efficient semi-structured sparsity-aware training framework for LLMs that jointly optimizes masks and weights to maintain model performance after sparsification.}

Building on this objective, our preliminary conference work presented at AAAI-25 introduced Adaptive Sparse Trainer (AST) \cite{huang2024pruninglargelanguagemodels}, a method based on the Straight-Through Estimator (STE) \cite{bengio2013estimatingpropagatinggradientsstochastic} that enables joint optimization of masks and weights for pretrained large language models. In particular, STE employs a sparse forward process by masking model weights and treats the non-differentiable masking operations as identity functions during backpropagation. This approximation provides gradients for masked parameters, thereby allowing on-the-fly mask updates based on the parameters' changing magnitudes. While AST achieved promising initial results, we conduced further analysis that revealed several limitations of the STE-based approach. Specifically, its discontinuous nature and reliance on approximate gradients can hinder effective sparsity optimization, often preventing the model from reaching its full potential.

To address the limitations of STE-based approaches and further enhance model precision, in this paper, we introduce Continuous Adaptive Sparse Trainer (CAST), a novel sparsity-aware training framework for semi-structured sparse models that is fully continuous and differentiable. CAST maintains a dense weight pattern for forward propagation throughout training, while inducing gradual sparsification via adaptive $L_1$ decay. This design eliminates the need for STE-based approximations and enables accurate gradient updates. To effectively induce the desired sparsity pattern, CAST introduces AdamS, a sparsity-aware Adam-based optimizer. Given a pretrained dense model, AdamS can progressively induces sparsity by incorporating a designed $L_1$ decay scheduler. Compared with the classic Adam optimizer, AdamS introduces two key adjustments: (1) A proportional gradient-decay mechanism that applies adaptive decay strength across all weights, regardless of their individual sensitivity to regularization; and (2) Decoupled momentum computation, which separates the decay term from the loss gradient during first-order momentum updates, preventing oscillations around zero. Additionally, CAST introduces a trainable weight-scaling module to compensate for the reduction in weight magnitudes caused by $L_1$ decay. This module can be seamlessly folded into the final sparse weights, adding no runtime overhead during deployment. CAST also retains the knowledge distillation technique proposed in AST \cite{huang2024pruninglargelanguagemodels} to further enhance training efficiency. Finally, we conduct extensive experiments across diverse model architectures and varying computational budgets to comprehensively validate CAST’s effectiveness and versatility.

To sum up, building upon our previous conference paper AST \cite{huang2024pruninglargelanguagemodels}, this journal paper proposes the Continuous Adaptive Sparse Trainer (CAST), a novel and efficient training framework designed for semi-structured sparse models, significantly enhancing both performance and efficiency. Our contributions are summarized as follows:

\begin{itemize}
    \item \textbf{Superior Performance:} We propose CAST, a fully continuous and differentiable sparsity-aware training framework, achieving state-of-the-art performance under similar computational budgets. Notably, on LLaMA2-7B, our 2:4 sparse model incurs less than a 0.1 increase in perplexity while improving zero-shot accuracy by 0.36\% compared to the dense model, using only 2\% of the original pretraining tokens.

    \item \textbf{AdamS Optimizer:} CAST introduces AdamS, a novel Adam-based optimizer specifically designed for sparsity-aware training. AdamS induces uniform and gradual sparsity throughout training, ensuring stable and effective sparsity-aware training. Additionally, this approach maintains dense forward propagation, avoiding biased gradient updates and thereby improving overall performance.
    
    \item \textbf{Weight Scaling Module:} CAST counteracts the weight magnitude shrinkage due to regularization by introducing a learnable weight scaling module. Notably, this module can be seamlessly integrated into the final model weights, preserving the intended sparsity pattern without additional inference overhead.

    \item \textbf{Knowledge Distillation:} CAST demonstrates that incorporating knowledge distillation enables a more efficient sparsity-aware training process and achieves superior performance under a similar computational budget.
    
    \item \textbf{Comprehensive Validation and Practical Applicability:}  We validate CAST across multiple model families, including GPT-2, OPT, LLaMA-2 and LLaMA-3, with parameter sizes ranging from 125M to 13B. Our comprehensive experiments, including scaling law, fine-tuning, quantization, and practical speedup evaluations, demonstrate CAST’s superior performance and practical potential for real-world applications in deploying semi-structured sparse models.
\end{itemize}

\section{Related Work}
Weight sparsity is a widely used model compression technique that reduces memory usage and computational cost, thereby accelerating inference. Existing methods generally fall into two categories: training-free (one-shot) pruning and training-based approaches. In the following sections, we review both in detail.
\subsection{One-shot Pruning}
 One-shot pruning methods trace back to early works such as Optimal Brain Damage (OBD) \cite{lecun1989optimal} and Optimal Brain Surgeon (OBS) \cite{hassibi1993optimal}. Modern training-free pruning techniques aim to avoid computational overhead by applying pruning without additional training. These methods are classified into pruning types based on their sparsity patterns: unstructured, structured, and semi-structured pruning. Unstructured pruning removes individual weights \cite{han2015deep, paul2022unmasking}, often maintaining strong performance even at high sparsity levels. However, due to the irregularity of the sparsity pattern, such models are difficult to accelerate efficiently on standard hardware. In contrast, structured pruning \cite{liu2017learning, molchanov2019importance, nova2023gradient, shen2022structural} removes entire units such as neurons, filters, or attention heads, making models more hardware-friendly but often at the cost of significant performance degradation. Semi-structured pruning, such as N:M sparsity \cite{hubara2021accelerated}, offers a middle ground by enforcing regular sparsity patterns that strike a balance between acceleration and performance. Recently, numerous studies \cite{frantar2023sparsegpt, sun2023simple, zhang2024plug, zhang2023dynamic} have made progress in pruning large language models (LLMs) with billions of parameters using training-free approaches. However, these pruned models often fall short of matching the performance of their dense counterparts on complex language understanding and reasoning tasks.

\subsection{Training Based Pruning Methods}

Another line of research \cite{NIPS2015_ae0eb3ee,  renda2020comparing, NEURIPS2020_d1ff1ec8, zhou2023three} focuses on training pruned models to recover performance. While substantial advances have been made on vision and smaller language models \cite{kurtic2022gmp, zhu2017prune}, many of these methods require repetitive training cycles \cite{DBLP:conf/iclr/FrankleC19} or introduce additional parameters during pruning \cite{shi2023upop}, limiting their scalability to LLMs. Some methods
pretrain sparse models using gradient estimators such as SR-STE \cite{zhou2021learning}, while later works improve efficiency by pruning a pretrained dense model and retraining either the weights or the masks. Sheared LLaMA \cite{xia2023sheared} employs a two-stage structured pruning process that results in models outperforming others of comparable size, showcasing the potential of retraining pruned networks. In the context of semi-structured sparsity, recent methods such as Wanda \cite{sun2023simple} first prune the model and subsequently updates only the remaining weights, whereas MaskLLM \cite{NEURIPS2024_0e9a05f5} utilizes a reparameterization strategy to learn masks via gradient updates without explicitly optimizing weights. However, these approaches generally fail to jointly optimize both masks and weights, and more importantly, they perform pruning in an abrupt, one-shot manner, limiting model adaptability. In contrast, our method jointly optimizes masks and weights and adopts a continuous transformation process, gradually converting dense models into sparse ones during training, facilitating smoother adaptation and better performance.

\section{Preliminary}

In this section, we formalize the objective of sparsity-aware training. Given a pretrained language model with $ L $ linear layers, we denote its set of parameters as $ \mathbf{\Theta} = \{W^0, W^1, \dots, W^{L-1}\} $, where each $ W^k \in \mathbb{R}^{R_k \times C_k} $ is the weight matrix of the $ k $-th linear layer, for $ k \in \{0, 1, \dots, L-1\} $.

Introducing sparsity amounts to applying an element-wise product between the original weight matrices and a set of binary masks of the same shape. The resulting sparse weight matrices $ \hat{W}^k \in \hat{\mathbf{\Theta}} = \{\hat{W}^0, \hat{W}^1, \dots, \hat{W}^{L-1}\} $ are given by:
\begin{equation} \label{eq:1}
\hat{W}^k = W^k \odot M^k,
\end{equation}
where $ M^k \in \mathbf{M} = \{M^0, M^1, \dots, M^{L-1}\} $ is the corresponding binary mask matrix for $W^k$, indicating whether each parameter is retained (1) or pruned (0). For notational completeness, we also use $\theta \in \mathbf{\Theta}$ to refer to individual scalar weights within the parameter set, where $ \theta $ denotes an entry in one of the matrices $ W^k $. The corresponding binary indicator for each scalar weight is denoted as $ m \in \mathbf{M} $.

\begin{figure}[t]
    \centering
    \includegraphics[width=3.5in]{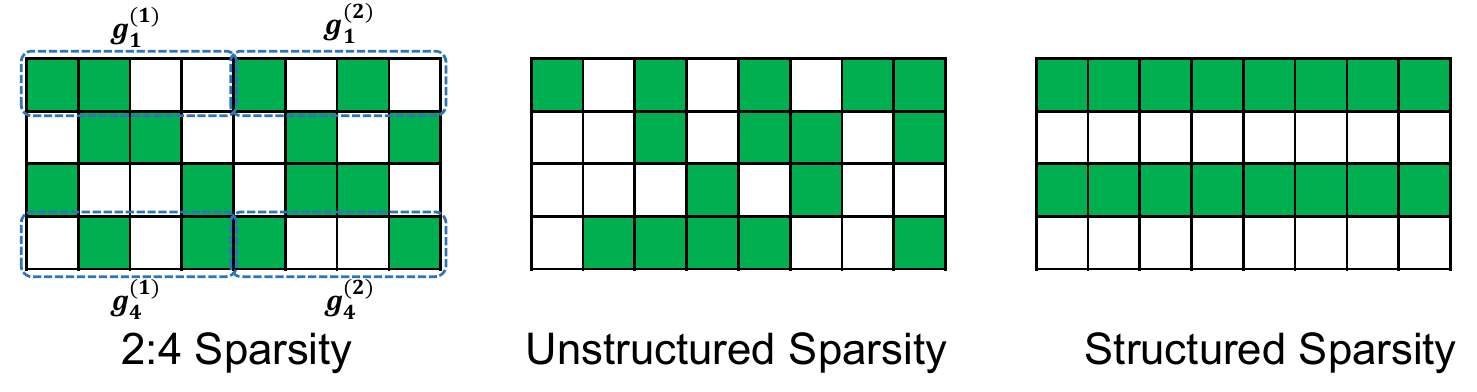}
    \caption{Illustration of different sparsity patterns under 50\% sparsity ratio. Unstructured sparsity constrains only the total number of non-zero elements; structured sparsity typically removes entire rows or columns; whereas 2:4 sparsity enforces that each 4-element group $g_r^{(i)}$ retains exactly 2 non-zero weights.} 
    \label{fig:1}
\end{figure}

\begin{figure*}[t]
    \centering
    \includegraphics[width=18.0cm,height=5.8cm]{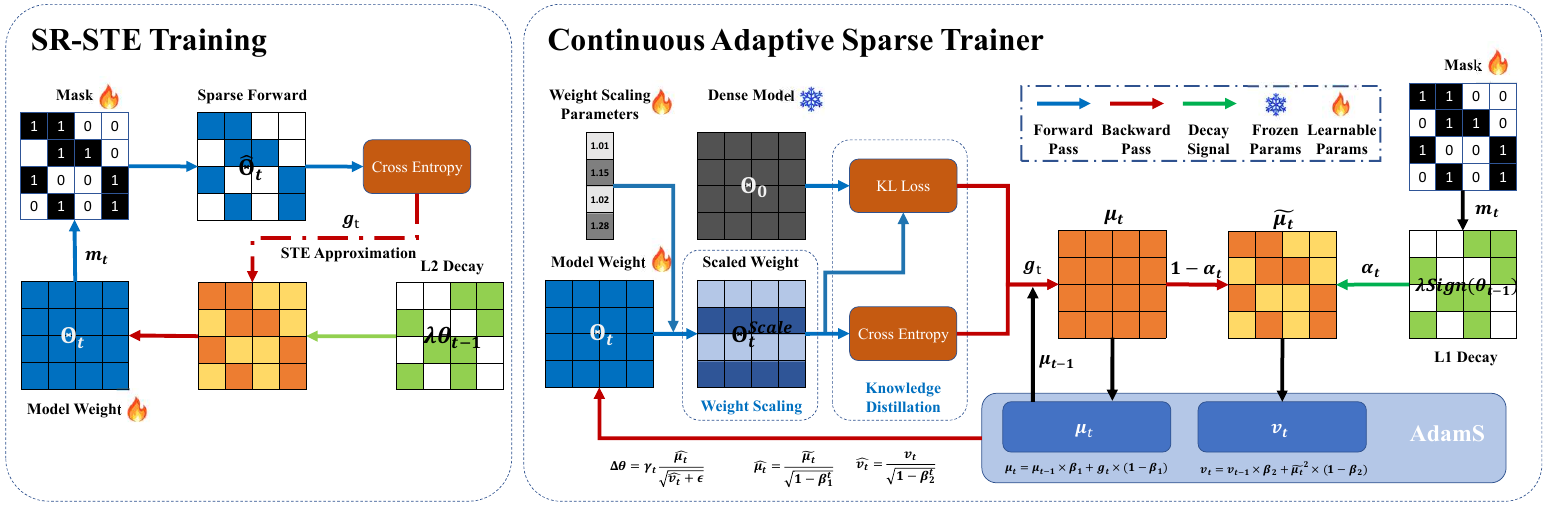}
    \caption{\textbf{(Left)} In the SR-STE training baseline, the forward process is sparse, which can lead to suboptimal performance due to gradient approximation and discontinuities. \textbf{(Right)} In contrast, Continuous Adaptive Sparse Trainer maintains a dense forward process while using AdamS to gradually induce sparsity. CAST further incorporates knowledge distillation and a learnable weight scaling module to enhance training efficiency and improve model performance.} 
    \label{fig:2}
\end{figure*}

The sparsity-aware training objective for N:M sparsity is to optimize $\mathbf{\hat{\Theta}}$ under the sparsity constraint $S(N,M,\mathbf{\Theta})$:
\begin{align}
\label{eq:2}
\min_{\mathbf{\hat{\Theta}} \in S(N,M,\mathbf{\Theta})} \quad & \mathbb{E}_{x \sim p(x)} \, \mathcal{L}(\mathbf{\hat{\Theta}}, x), 
\end{align}
where $x$ is the input, $p(x)$ is the data distribution and $\mathcal{L}$ is the language modeling loss. For $\hat{W}^k \in \mathbf {\hat{\Theta}}$ that satisfies the constraint $S(N,M,\mathbf{\Theta})$, each row $r \in \{0, 1, \dots, R_k - 1\}$ of $\hat{W}^k$ is partitioned into $C_k / M$ contiguous groups of size $M$. The column indices of the $i^{th}$ group of row $r$ are denoted as:
\begin{equation} \label{eq:3}
g_{r}^{(i)} = \{iM, iM+1, \dots, (i+1)M - 1\}, 
\end{equation}
where $i = 0, 1, \dots, C_k/M - 1$. The constraint requires that each group should retain exactly $ N $ non-zero elements. The corresponding mask $ M^k $ satisfies:
\begin{equation} \label{eq:4}
\sum_{j \in g_{r}^{(i)}} M^k[r, j] = N.
\end{equation}

In this paper, we train our models to follow the 2:4 sparsity pattern, where each group of four weights contains exactly two non-zero elements to align with hardware support. We illustrate the 2:4 sparsity pattern alongside other sparsity formats within a 4×8 matrix in Figure~\ref{fig:1}.

\section{Method}

In this section, we present Continuous Adaptive Sparse Trainer (CAST), a sparsity-aware training framework designed to effectively transform dense pretrained models into sparse counterparts while preserving their performance. CAST consists of three key components: First, we introduce AdamS, a sparsity-aware optimizer that gradually induces sparsity through selective regularization. Second, CAST proposes a learnable, sparsity-preserving weight scaling module that offsets the magnitude reduction caused by regularization. Third, CAST applies the knowledge distillation strategy from AST to enhance training stability and accelerate convergence. The overall training pipeline is illustrated in Figure~\ref{fig:2}.

\subsection{AdamS: A Sparsity Inducing Optimizer}

Given a pretrained dense model, we aim to solve the sparsity-constrained optimization problem defined in Equation~\eqref{eq:2}. To align the model weights with the desired sparsity pattern, a straightforward approach is to incorporate an additional $L_1$ regularization term into the language modeling loss, resulting in a Lasso-style \cite{51791361-8fe2-38d5-959f-ae8d048b490d} objective of the form:
\begin{equation} \label{eq:5}
\min_{\mathbf{\Theta}} \mathcal{L}(\mathbf{\Theta}, x) + \lambda \left\lVert \mathbf{\Theta} \right\rVert_1,
\end{equation}
where $\lambda$ is the regularization strength. However, this formulation faces critical limitations in the context of semi-structured sparsity. First, it provides no control over the resulting sparsity pattern or ratio, making it unsuitable for patterned formats such as 2:4 sparsity. Second, the direct competition between the primary loss and the regularization term can lead to degraded model performance. Finally, how to integrate standard $L_1$ regularization with momentum-based optimizers like Adam to support effective sparsity induction remains unexplored.

To this end, we introduce AdamS—a sparsity-aware optimizer that effectively transforms a pretrained dense model into a fine-grained 2:4 sparse one. Instead of explicitly computing the $L_1$ regularization term, AdamS applies it as a weight decay term integrated directly into the optimizer to improve efficiency. We present its pseudo-code in Algorithm~\ref{alg:AdamS}. Compared to the standard Adam optimizer, AdamS further introduces three key modifications to induce semi-structured sparsity.

First, AdamS employs selective decay by maintaining a set of binary masks at each iteration $t$, denoted as $\mathbf{M}_t = \{M_t^0, M_t^1, \dots, M_t^{L-1}\}$, where mask $M_t^k$ corresponds to the linear weight matrix $W_t^k \in \mathbf{\Theta}_t$ and conforms to the target 2:4 sparsity pattern defined in Equation \eqref{eq:4}. These masks are updated every $T_1 = 10$ iterations during training based on weight magnitude. Specifically, for weight matrix $W^k_t$, the mask values for each group $g_{r}^{(i)}$ are determined as follows:

\begin{align}
    \label{eq:6}
    M^k_t[r, j] = \begin{cases}
        1, & \text{if} \quad |W^k_t[r, j]| \geq \xi \\
        0,  & \text{if} \quad |W^k_t[r, j]| < \xi
    \end{cases} \quad j \in g_{r}^{(i)},
\end{align}
where $\xi$ is the second largest absolute value within the 4 elements of $g_{r}^{(i)}$. The resulting masks are then used to determine which weights are subjected to regularization at the current iteration. Accordingly, for each scalar parameter $\theta_t \in \mathbf{\Theta}_t$, $L_1$ decay is applied only to the masked parameters (i.e., those with $m_t = 0$). By introducing a decay term of $\lambda\text{sign}(\theta_t)$ alongside the gradient signal, $L_1$ decay gradually drives masked weights toward zero and thus aligning the dense model with the target 2:4 sparsity pattern encoded in the masks. The final sparse model is obtained by applying an element-wise product between the trained weights and the binary masks at the end of training. Notably, while AdamS performs dense forward computation throughout training, the masked weights are reduced to negligible magnitudes by the end, ensuring that the final hard pruning step introduces no performance degradation. 

\begin{algorithm}
\caption{AdamS, our proposed algorithm for sparsity-inducing optimization.}
\label{alg:AdamS}
\textbf{Input:}  Total training iterations $T$; mask update frequency $T_1$; regularization strength $\lambda$; exponential decay rates $\beta_1$, $\beta_2$; Given current iteration $t$, we denote the learning rate as $\gamma_t$; model weights as $\mathbf{\Theta}_t$; for each parameter $\theta_t \in \mathbf{\Theta}_t$, let $m_t \in \mathbf{M}_t$ denote its binary mask, $\mu_t$ its first-order moment, and $v_t$ its second-order moment.
\begin{algorithmic}[1] 
\FOR{each $\theta_{0} \in \mathbf{\Theta}_0$}
\STATE Initialize: $\mu_0$ = 0, $v_0=0$
\ENDFOR
\STATE $\mathbf{\Theta}_0$ as pretrained model parameters.
\FOR{$t=0,1,\dots,T-1$}
\IF{t$\mod$$T_1$ = 0}
\STATE \leavevmode\colorbox{yellow!40}{\parbox{\dimexpr\linewidth-2\fboxsep}{ Update mask {$\mathbf{M}_t$} \strut ---\textbf{Equation \eqref{eq:6}}}}

\ENDIF

\FOR{ each parameter $\theta_{t} \in \mathbf{\Theta}_t$}
\STATE $g_t \leftarrow \nabla_{\theta} \mathcal{L}(\theta_{t-1})$
\STATE $\mu_t \leftarrow \beta_1  \mu_{t-1} + (1-\beta_1)  g_t$
\STATE $\alpha_t = t/T$
\IF{$m_{t}=0$ }
\STATE \leavevmode\colorbox{cyan!20}{\parbox{\dimexpr\linewidth-2\fboxsep}{ $\tilde\mu_t \leftarrow (1 - \alpha_t) \mu_t + \alpha_t  \lambda  \operatorname{Sign}(\theta_{t-1})$ \strut -\textbf{Equation \eqref{eq:8}}}}
\ELSE
\STATE $\tilde \mu_t \leftarrow \mu_t$
\ENDIF
\STATE \leavevmode\colorbox{green!20}{\parbox{\dimexpr\linewidth-2\fboxsep}{ $v_t  \leftarrow \beta_2  v_{t-1} + (1-\beta_2) \tilde\mu_t^2$ \strut ---\textbf{Equation \eqref{eq:7}}}}

\STATE $\hat{\mu_t} \leftarrow \tilde\mu_t/(1-\beta_1^t) $
\STATE $\hat{v_t} \leftarrow v_t/(1-\beta_2^t) $
\STATE $\theta_{t} \leftarrow \theta_{t-1} - \gamma_t \cdot \hat{\mu_t} / (\sqrt{\hat{v_t}}+\epsilon)$
\ENDFOR
\ENDFOR

\FOR{each parameter $\theta_T \in \mathbf{\Theta}_T$}
\STATE Prune model by element-wise product  $\hat{\theta_T} = \theta_T \cdot m_T$
\ENDFOR
\RETURN $\mathbf{\hat{\Theta}_T}$
\end{algorithmic}
\end{algorithm}

Secondly, to ensure that masked weights are effectively decayed to zero, AdamS adopts a proportional decay strategy rather than directly adding the decay signal to the loss gradient. The decayed gradient signal is calculated as:
\begin{equation}
\label{eq:7}
\tilde{\mathcal{G}}_t = (1 - \alpha_t) \mathcal{G}_t + \alpha_t \lambda \operatorname{Sign}(\theta_{t-1}),
\end{equation}
where $T$ denotes the total number of training steps, and $\alpha_t = \frac{t}{T} \in [0, 1]$ is a time-dependent scaling factor. Here, $\mathcal{G}_t$ represents the gradient signal, which may correspond to the raw gradient in SGD or the first-order momentum in Adam. This formulation uses the coefficient $\alpha_t$ to dynamically adjust the relative contributions between gradient and decay components, ensuring that masked weights receive sufficient decay regardless of gradient influence, thereby enabling stable convergence to zero by the end of training.

Finally, we modify the update rule further to ensure compatibility between $L_1$ decay and momentum-based optimizers. Specifically, we apply the decay to the first-order momentum and use the resulting sum to compute the second-order moment. At iteration $t$, AdamS updates the parameter $\theta$ as:
\begin{align}
\label{eq:8}
     \mu_t  &= \beta_1  \mu_{t-1} + (1-\beta_1) g_{t}, \notag \\
    \tilde \mu_t &= (1 - \alpha_t) \mu_t + \alpha_t  \lambda  \operatorname{Sign} (\theta_{t-1}), \quad \hat\mu_t = \tilde\mu_t/(1-\beta_1^t), \notag \\
    v_{t}  &= \beta_2 v_{t-1} + (1-\beta_2) \tilde\mu_{t}^2, \quad \hat v_t = v_t/(1-\beta_2^t),  \notag \\
    \theta_{t} &= \theta_{t-1} -\gamma_t \frac{\hat\mu_t}{\sqrt{\hat v_t}+ \epsilon},
\end{align}
where $g_t$ denotes the gradient of $\theta$ with respect to the training loss $\mathcal{L}$; $\mu_t$ and $v_t$ represent the first-order and second-order moments, respectively; $\gamma_t$ is the learning rate; and $\beta_1$, $\beta_2$ are beta parameters. This design stems from the discontinuous nature of $L_1$ decay at zero. AdamS explicitly decouples the decay from the first-order momentum to ensure that the decay signal remains accurate and uninfluenced by historical information. This helps maintain a reliable decay direction, thereby effectively driving the masked weights toward zero.

With these modifications, AdamS supports agile joint optimization of both masks and weights throughout training, while ensuring that masked weights decay toward zero and conform to the desired sparsity pattern. We elaborate on the rationale behind these design choices in the following sections.

\subsubsection{Dynamic Mask Optimization}

\begin{figure}[t]
\centering
\includegraphics[width=3.2in]{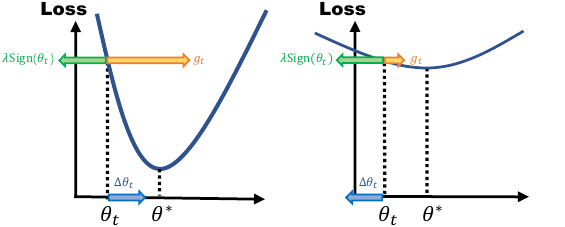}
\centering
\caption{(Left): Parameters with greater impact on model performance are preserved despite decay.
(Right): Parameters with minimal contribution to the loss are gradually driven to zero, allowing AdamS to adaptively identify and retain important weights during training.
}
\label{fig:3-5}
\end{figure}

AdamS uses weight magnitude as a proxy to gradually identify an optimal sparsity pattern while transforming a dense model into a 2:4 sparse one. As shown in Figure \ref{fig:3-5}, since important parameters typically receive consistently strong gradient signals due to their greater contribution to the loss, they tend to grow in magnitude despite regularization and are therefore retained. In contrast, less important parameters that contribute less to the loss are influenced by weak or noisy gradients and tend to decay toward zero. When combined with frequent and agile updates to the binary mask $\mathbf{M}_t$, this mechanism allows the sparsity pattern to adapt in real-time to the evolving importance of weights during training, effectively preserving salient weights while eliminating redundant ones. Importantly, by retaining dense forward computation and integrating sparsity induction into optimization process , AdamS avoids the performance degradation associated with explicit pruning in prior methods.

\subsubsection{Proportional Regularization}

\begin{figure}[t]
\centering
\includegraphics[width=3.6in]{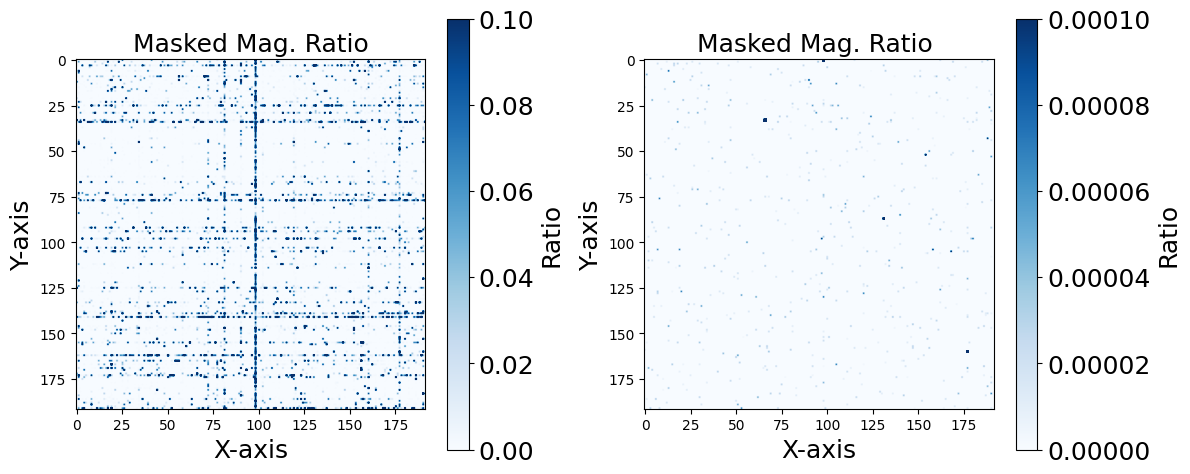}
\centering
\caption{(Left): Distribution of magnitude ratios of the two masked weights within 4-element groups under 2:4 sparsity in the first MLP layer of block 11 in the GPT-2 model with additive decay. (Right): Same distribution under proportional decay. Proportional decay uniformly drives masked weights to zero, whereas additive decay often leaves many masked weights insufficiently reduced, risking degradation during final pruning. Max pooling is applied to enhance readability.
}
\label{fig:3}
\end{figure}

When adding a uniform decay term directly to the gradient update, as in traditional approaches, we observe varying sensitivity across different parameters. Specifically, under the same regularization strength, some masked weights decay to near zero early in training (e.g., within the first 5\% of total steps), while others retain magnitudes comparable to unmasked parameters even at the end of training, as illustrated in the left graph of Figure \ref{fig:3}. This variability makes it challenging to select a single hyperparameter that performs consistently across all weights, often leaving a substantial portion insufficiently decayed. As a result, converting the final model to a strict 2:4 format post-training often results in significant performance degradation. To address this, we adopt a proportional regularization strategy as shown in Equation \eqref{eq:7}. Unlike directly adding decay to gradients, this proportional strategy provides more precise control over the relative influence of decay versus gradient signals, thereby allowing all weights to decay uniformly toward zero despite varying gradient signal strength. The method is also easy to tune in practice, as the hyperparameter $\lambda$ is set to the same order of magnitude as the gradient $g_t$. As shown in Figure~\ref{fig:3}, our approach uniformly drives all masked weights toward zero, regardless of their sensitivity to regularization, ensuring final pruning does not degrade performance.

\subsubsection{Decoupling $L_1$ Decay with First-Order Momentum} 
To better understand the challenges associated with applying selective \texorpdfstring{$L_1$}{l1} decay, we begin by examining how traditional Adam and AdamW optimizers incorporate decay, followed by a discussion of the limitations each approach presents in our setting. Specifically, the AdamW approach applies decay directly to the parameters after the standard gradient-based update. Its update rule is given by:
\begin{align}
\label{eq:9}
    \mu_{t}  &= \beta_1 \mu_{t-1} + (1-\beta_1) g_{t}, \quad \hat\mu_t = \mu_t/(1-\beta_1^t),\notag \\
    v_{t}  &= \beta_2 v_{t-1} + (1-\beta_2) g_{t}^2, \quad \hat v_t = v_t/(1-\beta_2^t) ,\notag \\ 
     \theta_{t} &= \theta_{t-1} -\gamma_t (\frac{\hat\mu_t}{\sqrt{\hat v_t}+ \epsilon} + \lambda \operatorname{Sign}(\theta_{t-1})).
\end{align}
 The Adam approach, on the other hand, integrates the decay term directly into the gradient, allowing it to be modulated by the second-order moment estimate $\sqrt{\hat{v}_t}$ before being applied to the weights. The corresponding update rule is given by:
\begin{align}
\label{eq:10}
    \tilde g_{t} & = g_{t}+\lambda \operatorname{Sign}(\theta_{t-1}), \notag \\ 
    \mu_{t}  &= \beta_1 \mu_{t-1} + (1-\beta_1) \tilde g_{t}, \quad \hat\mu_t = \mu_t/(1-\beta_1^t),\notag \\
    v_{t}  &= \beta_2 v_{t-1} + (1-\beta_2) {\tilde g_{t}}^2, \quad \hat v_t = v_t/(1-\beta_2^t), \notag \\
     \theta_t &= \theta_{t-1} -\gamma_t (\frac{\hat\mu_t}{\sqrt{\hat v_t}+ \epsilon}) ,
\end{align}

In the AdamW setting, challenges arise in identifying optimal masks due to its use of a uniform decay rate across all parameters. Unlike Adam, which scales the decay term by the second-order estimate $\sqrt{\hat{v}_t}$—allowing weights with larger gradients to receive smaller decay penalties and thus better preserve important weights—AdamW applies the same decay strength to all weights, regardless of gradient magnitude. This uniformity limits the optimizer’s ability to make decisive distinctions between competing weights, thereby reducing training stability and hindering mask learnability. Conversely, the Adam approach encounters difficulty in driving parameters precisely to zero, as it integrates the decay term into the first-order moment. Due to the discontinuous nature of $L_1$ decay at zero, this integration causes delayed adaptation. When a parameter crosses zero, the decay direction should ideally reverse immediately. However, accumulated momentum causes a lag, leading to oscillations near zero. This behavior weakens sparsity induction and impairs overall model convergence.

AdamS addresses these challenges by applying the $L_1$ decay term to the first-order moving average and using the resulting sum for second-order moment estimation, as shown in Equation \eqref{eq:8}. This design avoids inaccuracies caused by integrating $L_1$ decay into the first-order moment, while enabling the decay term to be adaptively scaled by $\sqrt{\hat{v}_t}$ for improved mask learnability. The second-order moment, which governs the adaptive update magnitude, is then computed from the combined signal of the moving average and decay term, ensuring stable and effective optimization.

$\textbf{Remark.}$ AdamS updates the mask every 10 iterations, which introduces negligible computational overhead. In terms of memory, it only adds a binary mask, which increase the optimizer state by just 1/32. While this work primarily focuses on semi-structured patterns, our method is inherently agnostic to sparsity formats and can be readily extended to others.
\subsection{Weight Scaling}

In the previous section, we applied decay to model parameters to induce sparsity. This inevitably reduces overall weight magnitudes, which can diminish the model's expressive capacity and adversely affect performance. To address this issue, we introduce a trainable scaling module that adjusts the magnitude of each group in weight matrix to compensate for the effects of regularization. Prior work such as ST-3 \cite{10030853} proposed manually computing a scaling factor based on the ratio of previous and current magnitudes. However, due to the dense forward process adopted in our method, this strategy is not directly applicable. Additionally, manual tuning can introduce fluctuations in the scaling factor, leading to instability during training. Therefore, we instead allow the scaling factors to be learned jointly during training. Specifically, for a weight matrix $W^k \in \mathbb{R}^{R_k \times C_k}$, we apply row-wise scaling by defining a scaling factor vector $A^k \in \mathbb{R}^{R_k}$. The scaled weights used during forward propagation become:

\begin{align}
W^{\text{scale}}_k 
= \begin{pmatrix}
a_1 \cdot \textbf{w}_1 \\
a_2 \cdot \textbf{w}_2 \\
\vdots \\
a_{R_k} \cdot \textbf{w}_{R_k}
\end{pmatrix}
= \text{Diag}(A^k)\,W^k,
\label{eq:11}
\end{align}
where $\text{Diag}(A^k)$ is a diagonal matrix constructed from the scaling vector $A^k$, and $\textbf{w}_i \in \mathbb{R}^{1 \times C_k}$ denotes the $i^{\text{th}}$ row of $W^k$. Furthermore, to achieve finer-grained scaling, we can apply different scaling factors to parameter groups within each row. Specifically, suppose we partition each row into $n$ groups, assuming $n \mid C_k$. In this case, the scaling matrix becomes $A^k \in \mathbb{R}^{R_k \times n}$. To apply this group-wise scaling, we reshape $W^k$ and $A^k$. The scaled weights used during forward propagation become:

\begin{align}
\mathrm{reshape}\left( W^k \in \mathbb{R}^{R_k  \times C_k} \right) \rightarrow \tilde{W^k} \in \mathbb{R}^{(R_k \cdot n) \times (C_k/n)}, \notag\\
\mathrm{reshape}\left( A^k \in \mathbb{R}^{R_k \times n} \right) \rightarrow \tilde{A^k} \in \mathbb{R}^{R_k \cdot n}, \notag \\
\tilde{W}_{\text{scale}}^k
= \begin{pmatrix}
\tilde{a}_1 \cdot \tilde{\textbf{w}}_1 \\
\tilde{a}_2 \cdot \tilde{\textbf{w}}_2 \\
\vdots \\
\tilde{a}_{R_k \cdot n} \cdot \tilde{\textbf{w}}_{R_k \cdot n}
\end{pmatrix}
= \text{Diag}(\tilde{A^k})\,\tilde{W^k}.
\label{eq:12}
\end{align}
Here, $ \tilde{\textbf{w}}_i \in \mathbb{R}^{1 \times (C_k / n)} $ represents the $ i^{\text{th}} $ reshaped row segment of $ \tilde{W}^k $. Finally, we reshape $ \tilde{W}_{\text{scale}}^k $ back to $ {W}_{\text{scale}}^k \in \mathbb{R}^{R_k \times C_k} $ for later operations. During our experiments, we initialize the scaling parameters to 1, allowing them to gradually increase and effectively compensate for the decreasing magnitudes of masked weights caused by regularization. The hyperparameter $n$ should be chosen with care: larger values may introduce redundancy, while smaller values may lack sufficient granularity. Importantly, since the scaling is applied via element-wise multiplication, the sparsity pattern is preserved, and the scaling factors can be seamlessly folded into the original weights. This design avoids the need for storing or computing additional parameters, incurring no inference-time overhead. 

\subsection{Improving Training Efficiency through Knowledge Distillation}

Given that most state-of-the-art language models do not release their original pretraining datasets, reproducing their performance under sparsity constraints presents a significant challenge. A practical solution is to use the original dense model as a teacher for knowledge distillation. This enables performance retention without access to the original dataset. Specifically, we apply KL divergence loss \cite{kullback1951information}. Unlike the standard language modeling loss, KL-divergence measures the difference between two probability distributions, providing richer supervisory signals that help reduce overfitting during the early stages of training. As a result, we adopt the following loss function:
 \begin{align}
    \mathcal{L}_{kl} = D_{\text{KL}}&(P_{\text{t}} \parallel  P_{\text{s}}) 
    = \sum_{x} P_{\text{t}}(x) \log \frac{P_{\text{t}}(x)}{P_{\text{s}}(x)}, \nonumber \\
     \mathcal{L} &= \eta \mathcal{L}_{kl} + (1-\eta)\mathcal{L}_{ce}, 
     \label{eq:13}
 \end{align}
where $\mathcal{L}_{ce}$ is the cross-entropy loss and $P_{t}$ and $P_{s}$ are the probability distribution of the teacher and student models, respectively. Furthermore, although knowledge distillation introduces additional overhead, we find it to be cost-efficient. Under the same training budget, it yields significantly better results, especially for smaller models where performance can degrade substantially without distillation. Additionally, we explored various distillation strategies that incorporate intermediate-layer information. However, the results showed that imposing constraints on intermediate outputs adversely affects generalization and leads to suboptimal performance. Detailed results are provided in Appendix H.

\section{Comparison of Mask Optimization Dynamics in Prior SAT Methods}

In this section, we compare the mask optimization process of CAST with existing sparsity-aware training (SAT) approaches for large language models. While many prior methods have shown success in inducing sparsity, they often fall short in supporting joint optimization of the sparsity pattern and model weights. For example, methods such as Wanda \cite{sun2023simple} and SNIP \cite{lee2018snip} treat the sparsity mask as fixed throughout training, entirely foregoing mask optimization. Conversely, MaskLLM \cite{NEURIPS2024_0e9a05f5} fixes the model weights and only learns the mask. Some approaches, such as RIGL \cite{evci2020rigging} and DNS \cite{guo2016dynamic}, enable limited joint adaptation by iteratively pruning and regrowing weights based on heuristic importance scores during training; however, these methods often rely on coarse approximations that fail to accurately capture true weight significance. Among existing approaches, SR-STE \cite{zhou2021learning} achieves stronger performance by enabling more dynamic co-optimization of weights and masks. It leverages the straight-through estimator (STE) \cite{bengio2013estimatingpropagatinggradientsstochastic} to update masked weights and adjust the sparsity pattern on-the-fly. Specifically, for a weight $\theta_t$ at iteration $t$, SR-STE applies an element-wise product between the mask and the weight during the forward pass, i.e., $\hat{\theta}_t = m_t \cdot \theta_t$. Let $g(\cdot) = \nabla_\theta \mathcal{L}(\cdot)$ denote the gradient with respect to $\theta$. To enable updates to masked parameters during the backward pass, SR-STE bypasses the masking operation by using the gradient of $\hat{\theta}_t$ as a proxy for the gradient of $\theta_t$:

\begin{equation} \label{eq:14}
 g(\theta_t)  \approx g(\hat{\theta}_t).
\end{equation}
Moreover, SR-STE applies $L_2$ decay to masked weights to suppress mask fluctuations. Since SR-STE is limited to the SGD optimizer, the corresponding update rule is:
\begin{equation} \label{eq:15}
    \theta_{t} = 
        \theta_{t-1} - \gamma_t (g(\hat{\theta}_{t-1}) + \lambda \theta_{t-1}).
\end{equation}
We compare representative baselines on LLaMA2-7B under 2:4 sparsity in Table~\ref{table:1}, where SR-STE's performance underscores the importance of joint weight-mask optimization.

\begin{table}[!t]

\caption{Perplexity and Zero-shot Accuracy(\%) of Different Pruning Methods on LLaMA2-7B with 2:4 Sparsity}
\label{table:1}
\centering

\begin{threeparttable}
\begin{tabular}{ccc}
\toprule
 Method & \makecell{ Perplexity } &  Average  Zero-shot Acc  \\
 \midrule
\textbf{Dense} & 5.12  & 57.16 \\
\midrule

Wanda & 11.29  & 45.98 \\

MaskLLM & 6.72  & 52.09 \\

Wanda Retraining & 5.78  & 54.73 \\

SR-STE & 5.74  & 54.99\\
\bottomrule
\end{tabular}
\begin{tablenotes}
\small
\item \textit{Note:} We used a similar computational budget for all training-based methods.
\end{tablenotes}
\end{threeparttable}

\end{table}

Despite demonstrating a certain level of effectiveness, the underlying dynamics of mask optimization in SR-STE remain unexplored. Ideally, an optimal mask optimization process should support frequent mask updates early in training to explore favorable sparsity patterns, while ensuring stability and convergence toward the end. Moreover, the learned masks should prioritize retaining weights that contribute meaningfully to model performance. Guided by these objectives, we identify two key limitations in SR-STE that hinder stable and effective mask learning: sparse forward execution and $L_2$ decay. In the following sections, we compare SR-STE with our proposed method, CAST, and demonstrate that dense forwarding combined with the AdamS optimizer addresses these challenges.

\subsection{Discontinuity and Non-differentiability}

First, the sparse forwarding approach used in previous approaches including SR-STE can introduce discontinuities and non-differentiability, which in turn lead to instability during optimization. For example, we consider two consecutive weights at iteration $ t $, $ (\theta_t^{1}, \theta_t^{2}) = (1.01, 1.00) $, where under a 1:2 sparsity constraint, $ \theta_t^{2} $ is masked. The corresponding forward values are $ (\hat \theta_t^{1}, \hat\theta_t^{2}) = (1.01, 0) $. Now, suppose a small gradient update leads to $ (\theta_{t+1}^{1}, \theta_{t+1}^{2}) = (1.00, 1.01) $; the forward parameters become $ (\hat\theta_{t+1}^{1}, \hat\theta_{t+1}^{2}) = (0, 1.01) $. This abrupt change in the forward pass, despite a minor weight update, highlights the discontinuity introduced by masking and adds substantial difficulty to the optimization process. 

Moreover, SR-STE uses the straight-through estimator to approximate gradients for masked weights by treating them as active during backpropagation, thereby ignoring the true effect of masking on the loss. This approximation results in an estimation error, formalized in the following lemma.

\begin{figure}[t]
\centering
\includegraphics[width=3.5in]{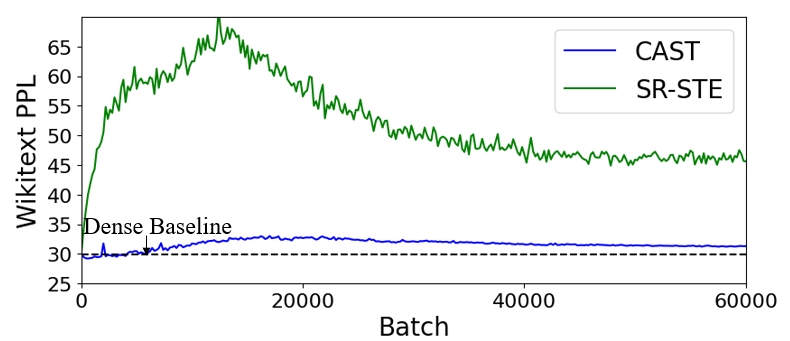}
\centering
\caption{Validation perplexity on WikiText for GPT-2 under dense forward pass of SR-STE and CAST respectively. In SR-STE based training, the dense model's performance deteriorates rapidly, indicating that masked weights become unreliable upon reactivation. In contrast, CAST maintains the dense model in a lossless state throughout retraining, enabling effective and meaningful mask learning.}
\label{fig:4}
\end{figure}

\begin{lemma}
The estimation error introduced by SR-STE grows linearly with the magnitude of the masked weights.
\end{lemma}
\begin{proof}
See Appendix G.
\end{proof}

As a result, the update directions for masked parameters may fail to effectively reduce the language modeling loss. To illustrate this, we track the perplexity under a dense forward pass during SR-STE training, as shown in Figure~\ref{fig:4}. The results reveal that, due to inaccurate gradient estimation, overall model performance deteriorates significantly over time, indicating that the masked parameters become increasingly unreliable for optimizing the language modeling objective as training progresses. As a result, even when these parameters are eventually unmasked, they fail to make meaningful contributions to model improvement, ultimately leading to suboptimal results. In contrast, CAST avoids the above issues by maintaining dense forwarding, enabling a continuous and differentiable optimization process.

\subsection{Bottleneck of $L_2$ regularization}

Furthermore, SR-STE adopts $L_2$ regularization, which decays weights proportionally to their magnitude. This formulation imposes stronger penalties on parameters with larger absolute values. As a result, the growth of high-magnitude masked weights is excessively suppressed, limiting the model’s ability to learn meaningful mask patterns for these potentially important parameters. In contrast, low-magnitude weights receive minimal decay, causing them to oscillate around zero rather than converging to stable masked or unmasked states. Consequently, most of the so-called "mask learning" in SR-STE occurs among these low-magnitude weights, leading to frequent yet uninformative mask updates. Instead of enabling decisive selection among competing high-magnitude weights, the model remains stuck in shallow oscillations that ultimately hinder convergence. In contrast, CAST avoids this issue by adopting $L_1$ regularization, a sparsity-inducing decay mechanism that uniformly drives undesired weights toward zero. This promotes clearer separation between important and unimportant parameters, enabling more effective mask learning, particularly for large-magnitude weights.

To illustrate the difference, we track a randomly selected subset of model weights during training and monitor their statistics. As shown in Table \ref{table:2}, we report three key metrics: Avg. Param Magnitude, representing the average magnitude of unmasked parameters; Avg. Mag. @ Last Flip, indicating the average magnitude of weights when their mask status last changed; and Prog. @ Last Flip, denoting the average training progress percentage at which the final mask update occurred. The results show that, compared to our method, SR-STE training exhibits frequent mask oscillations throughout the training process rather than converging early. Moreover, most mask updates in SR-STE occur on weights near zero, indicating that the method fails to make decisive and meaningful selections among large-magnitude weights and instead oscillates around less important, small-magnitude ones. Further details on mask learning behaviors can be found in the Appendix C.
\begin{table}[!t]
\renewcommand{\arraystretch}{1.3}
\caption{Mask Learning Statistics During SR-STE and CAST Retraining}
\label{table:2}
\centering
\begin{threeparttable}
\begin{tabular}{c|ccc}
\toprule
 Method & \makecell{ Avg. \\ Params Mag.} & \makecell{ Avg. Mag. \\ @ Mask Flip} & Prog. @  Last Flip\\
\midrule
SR-STE & 0.095 & 0.006 & 46.55 \% \\
\midrule
CAST & 0.142 & 0.106 & 16.87\% \\
\bottomrule
\end{tabular}

\begin{tablenotes}
\small
\item \textit{Note:} In CAST, we maintain a dense forward pass and apply decay to masked weights during backward propagation. Importantly, average weight magnitudes are computed only for unmasked parameters for both methods.
\end{tablenotes}
\end{threeparttable}

\end{table}

\section{Experiments}

We evaluate CAST across multiple model families, comparing it against both training-free and training-based state-of-the-art methods. Below, we summarize the datasets, baselines, evaluation metrics, and key results. In addition, we perform ablation studies to assess the individual contributions of each component within CAST. We further conduct a series of experiments to address the following key questions. First, how effectively does CAST’s performance scale under extended training token budgets? Second, can post-training quantization further enhance the compression efficiency of CAST without compromising accuracy? Third, how well do CAST-trained sparse models transfer to downstream fine-tuning tasks? Finally, what practical end-to-end inference speedups can CAST achieve in practical deployment scenarios? Detailed results and analyses are provided in the sections below.

\subsection{Experiment Setup}

\begin{table*}[t]
\centering
\caption{Performance of Various Sparsity‑Aware Training Methods on Seven Zero‑Shot Tasks and WikiText Perplexity.}
\resizebox{0.95\textwidth}{!}{%
\begin{threeparttable}
\begin{tabular}{cc|cccccccc|c}
\toprule
\textbf{Methods} &
\makecell{Tokens \\Trained}  &
\textbf{HellaS.} &
\textbf{RACE} &
\textbf{PIQA} &
\textbf{WinoG.} &
\textbf{ARC-e} &
\textbf{ARC-c} &
\textbf{OBQA} &
\textbf{Average} &
\makecell{Wikitext \\PPL} \\
\midrule
LLaMA2-7B & 2T  & 57.03  & 44.11  & 78.07  & 69.38  & 75.38  & 42.92  & 33.20  & 57.16 & 5.12 \\ 
CAST$^\dagger$ & 40B & 56.13 & 40.86 & 77.58 & 69.53 & 77.78 & 47.18 & 33.60 & 57.52 & 5.21 \\ 
\midrule

Wanda & \ding{55} & 41.05  & 35.02  & 70.78  & 62.67  & 61.99  & 27.56  & 22.80  & 45.98 & 11.29 \\ 
MaskLLM & 2B & 50.91  & \textbf{40.77}  & 74.92  & 64.48  & 69.57  & 36.00  & 28.00  & 52.09 & 6.72 \\ 
Naive Retraining & 10B & 53.90 & 38.28 & 76.61 & 68.27 & 75.21 & 41.21 & 29.60 & 54.73 & 5.78 \\ 
SR-STE & 10B & 54.02  & 39.02 & 76.88 & \textbf{68.35}  & 75.58 & 41.46 & 29.60 & 54.99 & 5.74 \\ 
CAST & 7.5B & \textbf{54.50} & 40.48 & \textbf{77.09} & 68.27 &\textbf{76.52} & \textbf{43.68} & \textbf{30.80} & \textbf{55.91} & \textbf{5.58} \\ \midrule
LLaMA2-13B & 2T & 60.15  & 44.59  & 79.27  & 72.45  & 78.93  & 47.18  & 34.60  & 59.60 & 4.57 \\ 
CAST$^\dagger$ & 15B & 58.01 & 41.24 & 77.42 & 72.45  & 79.50 & 49.06 & 34.80 & 58.93 & 4.71 \\ \midrule
Wanda & \ding{55} & 46.96  & 38.09  & 74.05 &  66.69 & 68.64  & 34.81   & 25.00  & 50.61 & 8.47 \\ 
MaskLLM & 2B & 55.09  & \textbf{41.24}  & 77.69  & 67.80 & 73.15  & 40.44  & 30.00  & 55.06 & 5.85 \\ 
Naive Retraining & 10B & 58.08 & 39.43 & 78.07 & 71.35 & 77.23 & 47.77 & 33.60 & 57.93 & 5.08 \\
SR-STE & 10B & \textbf{58.27}  & 40.38 & \textbf{78.32} & 71.11 & 78.29  & \textbf{47.79} & 34.80 & 58.42 & 5.04 \\
CAST & 7.5B & 58.14 & 40.67 & 78.13 & \textbf{72.30} & \textbf{78.83} & \textbf{47.79} & \textbf{35.40} & \textbf{58.75} & \textbf{4.91} \\  

\midrule
LLaMA3-8B & 15T & 60.10  & 40.00  & 79.43 & 73.56  & 80.26  & 50.00  & 34.80 & 59.74 & 5.76 \\ 
CAST$^\dagger$  & 40B & 57.90 & 40.77 & 79.27 & 71.67 & 81.40 &50.43 & 34.60& 59.43 & 6.33 \\  \midrule

Wanda & \ding{55}  & 37.09 & 32.63 & 67.41 & 59.75 & 56.48 & 25.85 & 18.00 & 42.46 & 22.97 \\ 
MaskLLM & 2B & 53.89  & 37.79  & 77.86 & 67.88  & 71.88 & 38.73  & 30.00 & 54.00 & 8.50 \\ 
Naive Retraining & 10B & 56.07 & 39.43 & 78.56 & 70.48 & 78.78 & 46.41 & 32.20  & 57.42 & 7.26 \\ 
SR-STE & 10B & 56.18 & \textbf{39.81} & 78.89 & 70.64 & 78.57 & 47.18 & 31.80 & 57.58 & 7.22 \\
CAST & 7.5B & \textbf{56.41} & \textbf{39.81} & \textbf{79.05} & \textbf{71.74} & \textbf{79.54} & \textbf{48.89} & \textbf{33.20} & \textbf{58.38} & \textbf{6.85} \\ 

\bottomrule
\end{tabular}%
\begin{tablenotes}
\small
\item \textit{Note:} We ensure that all training-based methods consume approximately the same amount of computational resources, except for CAST$^\dagger$. All evaluation results are obtained from our own experiments, except for MaskLLM, which are reported directly from the corresponding papers. We use \textbf{bold} to indicate the best result.
\end{tablenotes}
\end{threeparttable}
}
\label{table:3}
\end{table*}

\textbf{Model Configurations.} We evaluate the performance of CAST across three open-sourced LLM families: LLaMA \cite{ grattafiori2024llama, touvron2023llama}, which includes both LLaMA2 and LLaMA3 models, GPT-2 \cite{brown2020language}, and the OPT \cite{zhang2022optopenpretrainedtransformer} models. All models are trained to follow a strict 2:4 semi-structured sparsity pattern to ensure practical speedup on modern hardware. We select $n=2$ for weight scaling module. Optimal hyperparameters are determined via grid search, with detailed configurations and training settings provided in Appendix B.

\textbf{Data.} To ensure comparability and isolate the effect of our method, we avoid using significantly stronger datasets than those used in original model pretraining. This ensures that performance gains are due to the method rather than differences in data quality. For LLaMA models, we adopt Dolmino-Mix-1124—an open-source, high-quality, domain-targeted dataset which achieves zero-shot performance comparable to the original LLaMA training corpus. For OPT and GPT-2 models, we use the weaker C4 dataset to remain consistent with their original pretraining data.

\begin{table*}[ht]
\centering
\caption{Accuracy(\%) on Various Few-Shot Tasks for the LLaMA Models. }
\resizebox{0.8\textwidth}{!}{%
\begin{tabular}{cccccccccc}
\toprule 
  Model & \makecell{Weight \\Pattern} & MMLU & MATH & TriviaQA & NQ & BoolQ & CSQA & SciQ & Average \\ \midrule
  \multicolumn{1}{c}{\multirow{2}{*}{LLaMA2-7B}}  &Dense & 45.74 &  5.56   &  \textbf{64.15} & \textbf{26.34} & 78.90 & 56.01 & \textbf{97.00} & 53.39\\
 &Sparse & \textbf{52.34}  &  \textbf{7.64} & 62.32 & 25.90 & \textbf{79.94} & \textbf{64.62} &  \textbf{97.00} & \textbf{55.68} \\ \midrule
 \multicolumn{1}{c}{\multirow{2}{*}{LLaMA2-13B}}  &Dense & 55.13 &  6.78   &  \textbf{70.45} & \textbf{30.69} & 83.21 & 67.57 & \textbf{97.50} & 58.76 \\
 &Sparse & \textbf{56.07}  &  \textbf{8.12} & 67.65 & 28.83 & \textbf{84.31} & \textbf{71.42} & 97.30 & \textbf{59.10}\\ \midrule 
 \multicolumn{1}{c}{\multirow{2}{*}{LLaMA3-8B}}  &Dense & \textbf{65.43} &  13.42   & \textbf{71.69}  & \textbf{29.39} & 81.93 & \textbf{73.62} & 97.60 & \textbf{61.87}\\
 &Sparse & 62.15  &  \textbf{13.44} & 65.84 & 27.53 & \textbf{82.66} & 73.38 & \textbf{97.80} & 60.40\\
 \bottomrule
\end{tabular}
}
 \label{table:4}
\end{table*}

\textbf{Baselines.} We compare our method against both training-free and training-based sparsification approaches. For training-free baselines, we include Wanda \cite{sun2023simple} and SparseGPT \cite{frantar2023sparsegpt}. Among training-based methods, we evaluate three representative approaches: (1) MaskLLM \cite{NEURIPS2024_0e9a05f5}, which learns weight masks through gradient-based optimization on open-source datasets; (2) naive retraining, which fine-tunes the model weights following one-shot pruning, without updating the masks. Notably, we find that the choice of one-shot mask selection strategy at the beginning has minimal impact on final performance. As magnitude-based pruning often outperforms more complex criteria like Wanda or SparseGPT while requiring less compute, we adopt it as the default in the retraining setting; and (3) SR-STE, originally designed for sparse pretraining, which we adapt for retraining by initializing it with dense pretrained models. To ensure fair comparison, all training-based methods are allocated similar compute budget. 

\begin{table*}[ht]
  \caption{Perplexity Results on WikiText‑2 Dataset for 2:4‑Sparsified Language Models. }
  \centering
\resizebox{0.8\textwidth}{!}{%
\begin{threeparttable}
\begin{tabular}{cccccccccc}
\toprule
\multicolumn{1}{l}{} &   
\multicolumn{1}{l}{} & 
  \multicolumn{1}{l}{} &
  \multicolumn{3}{c}{OPT} &
  \multicolumn{4}{c}{GPT2} \\ \cmidrule(lr){4-10}
Methods & Weight Updates &
 Tokens Trained &
  125M &
  350M &
  1.3B &
  124M &
  350M &
  774M &
  1.5B \\ \midrule
Dense &  -  &  - &27.76 &22.00 &14.62 & 29.95 & 21.72 & 19.43 & 17.40 \\ \midrule
Wanda  & \ding{55} & \ding{55}  &60.91 &50.16  &23.92   & 115.64  & 63.71  & 49.97   & 30.44  \\
SparseGPT & \ding{51} & \ding{55}   & 45.58  & 40.33 & 29.03 & 50.09  & 31.03  & 25.98  & 21.14  \\
Naive Retraining  & \ding{51} & 10B  &  31.74 &  25.86  &  17.60  & 40.58 & 26.85  & 22.69 & 20.80  \\
SRSTE & \ding{51} &  10B  &    32.17  &  25.18  & 17.02 & 38.30 & 23.74  & 21.43 & 19.13  \\
AST & \ding{51} &  7.5B  &    30.08  &  24.26  & 15.67 & 32.06 & 23.40  & 20.78 & 18.14  \\
CAST & \ding{51}  & 7.5B & \textbf{28.32} &  \textbf{23.57}  & \textbf{14.92} & \textbf{31.24} & \textbf{23.12} & \textbf{19.97} & \textbf{17.35}  \\  \bottomrule
\end{tabular} %

\end{threeparttable}
}
  \label{table:5}
\end{table*}

\textbf{Evaluation.} We evaluate all models using WikiText-2 perplexity. In addition, we evaluate zero-shot and few-shot performance on the LLaMA2 and LLaMA3 model families using EleutherAI’s LM Harness \cite{gao2021framework}. These tasks include ARC-Easy and ARC-Challenge \cite{Clark2018ThinkYH}, OpenBookQA \cite{mihaylov2018can}, WinoGrande \cite{sakaguchi2021winogrande}, PIQA \cite{Bisk2020}, HellaSwag \cite{zellers2019hellaswag}, RACE \cite{lai2017racelargescalereadingcomprehension}, CommonSenseQA \cite{talmor-etal-2019-commonsenseqa}, Science Questions \cite{welbl-etal-2017-crowdsourcing}, and BoolQ \cite{clark2019boolqexploringsurprisingdifficulty}. We also evaluate performance on knowledge-intensive tasks, including MMLU \cite{hendrycks2020measuring}, GSM8K \cite{cobbe2021training}, MATH \cite{hendrycksmath2021}, TriviaQA \cite{joshi2017triviaqalargescaledistantly}, and Natural Questions \cite{kwiatkowski-etal-2019-natural}.

\subsection{Main Results}

Table~\ref{table:3} reports the results for LLaMA models. Since our method incorporates knowledge distillation, which increases FLOPs by roughly $1/3$, we decrease its training tokens by a similar amount to match the compute budget. A detailed analysis of actual time complexity is provided in Appendix F. Across nearly all model settings, our method consistently outperforms prior approaches under similar computational budgets. For instance, on LLaMA2-7B, CAST reduces the perplexity by 0.16 and 0.20, and improves zero-shot accuracy by 0.92\% and 1.18\% over SR-STE and naive retraining, respectively. To evaluate the potential of our approach when more tokens are trained, we further conduct an extended training run within our computational capacity denoted by CAST$^\dagger$. With increased training tokens, our 2:4 sparse models match or even surpass their dense counterparts in both perplexity and zero-shot accuracy. 

To further assess our sparse models, we report few-shot results on challenging benchmarks in Table~\ref{table:4}, including knowledge-intensive tasks like MMLU and MATH. For models pretrained on weaker datasets (e.g., LLaMA2), our sparse models not only match but often surpass their dense counterparts, while using only a fraction of the original training tokens, partly due to a pretraining corpus more focused on knowledge-intensive tasks. However, even for models pretrained on stronger closed‑source datasets like LLaMA3, our method still introduces negligible performance degradation. These results suggest that factual knowledge and reasoning capabilities can be effectively recovered through sparse retraining. In contrast, perplexity typically requires more training tokens to fully recover.

Table~\ref{table:5} summarizes the results for GPT-2 and OPT models, which follow a similar trend. Our method consistently outperforms the baselines by a significant margin. Furthermore, we observe that performance recovery is generally easier for larger models across all three model families. For example, within the GPT-2 family, GPT2-XL performs similarly to its dense counterpart on the Wikitext dataset, while smaller variants like GPT2 exhibit a more pronounced performance gap. This suggests that larger models in the same model families may be undertrained relative to their model size, making them more resilient to compression. This trend aligns with findings in the quantization literature \cite{kumar2024scalinglawsprecision}, which show that undertrained models tend to degrade less under compression. Importantly, our 2:4 sparse models significantly outperform dense models with a similar number of non-zero parameters, highlighting the practical effectiveness of our method.

Finally, we include results from our preliminary work, AST, in Table~\ref{table:5}. We limit the comparison to GPT and OPT models, as AST was trained on a weaker dataset than the one used for LLaMA in CAST. While AST already demonstrates competitive performance against existing baselines on these models, our proposed method CAST consistently achieves the best results under strict 2:4 sparsity, highlighting its superior ability to induce semi-structured sparsity through a continuous and differentiable optimization process.

\begin{table*}[t]
\centering
\caption{Ablation Study on the Performance Impact of CAST Components Across Three Model Families.}
\label{table:6}
\resizebox{1.0\textwidth}{!}{%
\begin{threeparttable}
\begin{tabular}{cc|c|cccccccc|c|c}
\toprule
\multicolumn{1}{l}{} &   
\multicolumn{1}{l}{} &  
  \multicolumn{9}{c}{LLaMA2-7B} &
  \multicolumn{1}{c}{OPT-125m}  &
    \multicolumn{1}{c}{GPT2} \\ \midrule
\textbf{Methods} &
\makecell{Tokens \\Trained}  &
\makecell{Wikitext \\PPL} &
  \textbf{HellaS.} &
  \textbf{RACE} &
  \textbf{PIQA} &
  \textbf{WinoG.} &
  \textbf{ARC-e} &
  \textbf{ARC-c} &
  \textbf{OBQA} &
  \textbf{Average} &
  \makecell{Wikitext \\PPL} &
  \makecell{Wikitext \\PPL} \\ \midrule
CAST & 7.5B & 5.56 & 54.50 & 40.48 & 77.09 & 68.27 & 76.52 & 43.68 & 30.80 &  55.91   & 28.32 &  31.24 \\ \midrule
CAST w/o KD & 10B &  5.61  & 54.42 & 38.85 & 75.57 & 69.14 & 75.59 & 43.60 & 30.40 & 55.37  & 31.05 & 38.58\\ 

  {CAST w Fixed Mask} & 7.5B & 5.79 & 53.49 & 39.43 & 76.44 & 68.11 & 76.52 &41.89 &  28.60  & 54.93   & 29.90 & 32.74\\ 
{CAST w SRSTE} &7.5B& 5.67 & 54.60 & 38.76 & 77.20 & 68.03 & 41.89 & 76.39 & 28.80 & 55.10 & 29.45 & 32.61\\ 
\makecell{ CAST  \\ w/o Weight Scaling}  & 7.5B  & 5.58 & 54.78 & 40.29 & 76.09 & 67.96 & 76.09 & 43.00 & 30.60 & 55.54   & 28.43 & 31.43 \\
\midrule
{Naive Retraining} & 10B & 5.78 & 53.90 & 38.28 & 76.61 & 68.27 & 75.21 & 41.21 & 29.60 & 54.73 &   31.74& 40.58 \\
{ Naive Retraining  w KD}  & 7.5B & 5.79  &  53.47 & 39.43 & 76.33& 68.27 & 76.56   &  42.23& 28.80 & 55.01 & 29.54 & 32.93\\
\makecell{ Naive Retraining  \\ w KD \& AdamS}  & 7.5B  & 5.58 & 54.78 & 40.29 & 76.09 & 67.96 & 76.09 & 43.00 & 30.60 & 55.54   & 28.43 & 31.43 \\

 \bottomrule
\end{tabular}%
\begin{tablenotes}
\small
\item \textit{Note:} In our setting, "KD" refers to knowledge distillation. Methods without knowledge distillation employed additional tokens for fair comparison. "CAST w Fixed Mask" disables mask learning by fixing the model mask and omitting AdamS updates. "CAST w SR-STE" replaces AdamS with SR-STE for learning the sparse mask.

\end{tablenotes}

\end{threeparttable}
}
\end{table*}

\begin{figure*}[ht]
    \centering
    \begin{subfigure}[b]{0.23\textwidth}
        \includegraphics[width=\textwidth]{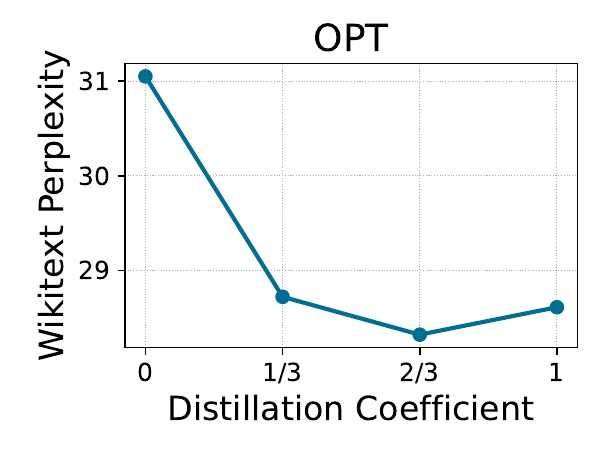}
        \caption{OPT Wikitext perplexity under different $\eta$.}
    \end{subfigure}
        \hfill
    \begin{subfigure}[b]{0.23\textwidth}
        \includegraphics[width=\textwidth]{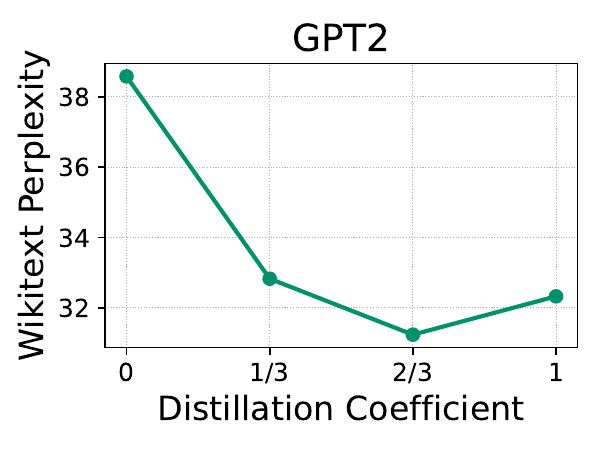}
        \caption{GPT2 Wikitext perplexity under different $\eta$.}
    \end{subfigure}
    \hfill
    \begin{subfigure}[b]{0.23\textwidth}
        \includegraphics[width=\textwidth]{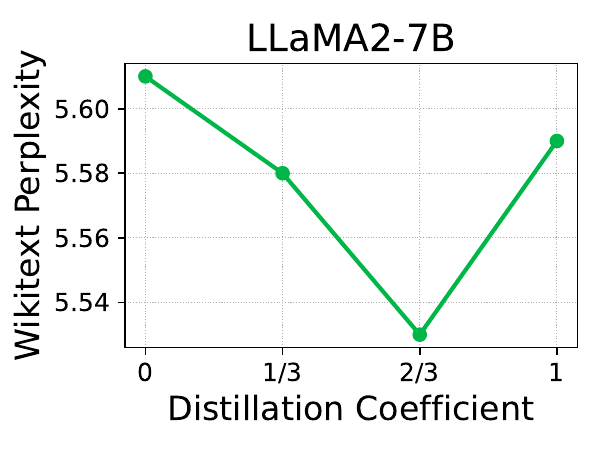}
        \caption{LLaMA2-7B Wikitext perplexity under different $\eta$.}
    \end{subfigure}
    \hfill
    \begin{subfigure}[b]{0.23\textwidth}
        \includegraphics[width=\textwidth]{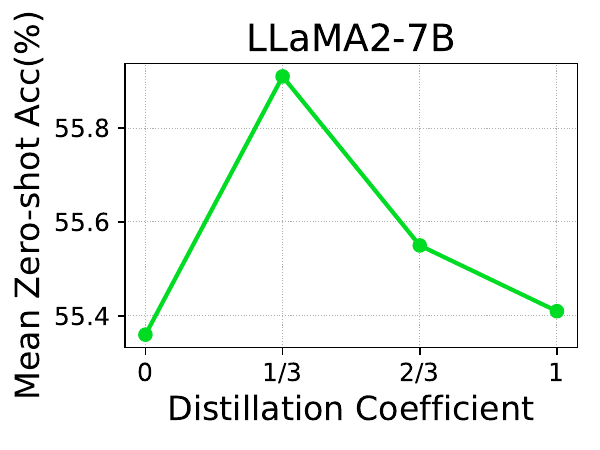}
        \caption{Accuracy of LLaMA2-7B under different values of $\eta$.}
    \end{subfigure}

    \caption{Performance of  OPT-125M, GPT2, and LLaMA2-7B under varying knowledge distillation coefficients ($\eta$). Note that for $\eta = 0$ (i.e., no teacher supervision), additional training steps are provided to match the computational cost.}
    \label{fig:5}
\end{figure*}

\subsection{Ablation Study}
\subsubsection{Ablation Study on Each Components of CAST}

In this section, we present an ablation study to assess the individual contribution of each component in CAST across three models from different families. Compared to naive retraining, CAST incorporates three key techniques: knowledge distillation, AdamS, and a weight scaling module. We conduct the ablation in two ways: first we remove each component from CAST individually to examine its impact. Notably, since AdamS is a mask learning method, for comparison we replace AdamS with alternative mask learning strategies, including fixed mask retraining and SR-STE. As shown in the top section of Table~\ref{table:6}. For the LLaMA2-7B model, we observe that removing each lead to respective performance drops of 0.54\%, 0.81\%, and 0.37\% in zero-shot accuracy, demonstrating the effectiveness of each component.

Secondly, we progressively add knowledge distillation, AdamS, and the weight scaling module on top of the naive retraining baseline and measure their effects. The results also demonstrate that all three components contribute significantly to performance improvements across a range of model architectures. As shown in the second section of Table~\ref{table:6}, on the LLaMA2-7B model, knowledge distillation, AdamS, and the weight scaling module improve zero-shot accuracy by 0.28\%, 0.53\%, and 0.37\%, respectively. This progressive enhancement demonstrates the complementary benefits of these components and further validates the design of CAST as a unified and effective retraining framework.

\subsubsection{Ablation Study on Distillation Coefficient}

We further conduct an ablation study on the distillation coefficient $\eta$ introduced in Equation~\eqref{eq:13} using the OPT-125M, LLaMA2-7B, and GPT-2 models. We vary $\eta$ from 0 to 1, where $\eta = 1$ corresponds to using only the distillation loss, and $\eta = 0$ corresponds to uses only the standard cross-entropy loss. For the non-distillation case ($\eta = 0$), we allocate additional training tokens to match the overall computational budget.

Figure~\ref{fig:5} presents both perplexity and zero-shot accuracy results. We observe that a higher distillation coefficient (e.g., $\eta = \frac{2}{3}$) consistently leads to lower perplexity across all models, confirming its effectiveness for optimizing language modeling objectives. However, smaller values of $\eta$ yield better zero-shot accuracy on LLaMA2-7B. We attribute this discrepancy to a mismatch between the teacher model’s output distribution and the target data distribution: the Dolmino-Mix dataset exhibits stronger zero-shot capabilities, while the LLaMA2-7B teacher model is more aligned with perplexity optimization. Varying the distillation strength shifts the learning focus between mimicking the teacher’s output distribution and fitting the data distribution, resulting in a tradeoff between perplexity and accuracy.

Moreover, relying exclusively on either distillation or cross-entropy loss results in suboptimal outcomes. In particular, using only cross-entropy—even with extended training—performs significantly worse. Based on these observations, we adopt $\eta = \frac{1}{3}$ in our extended training runs for the LLaMA models, striking a balance between convergence stability and generalization performance.

\begin{figure*}[ht]
    \centering
    \begin{subfigure}[b]{0.32\textwidth}
        \includegraphics[width=\textwidth]{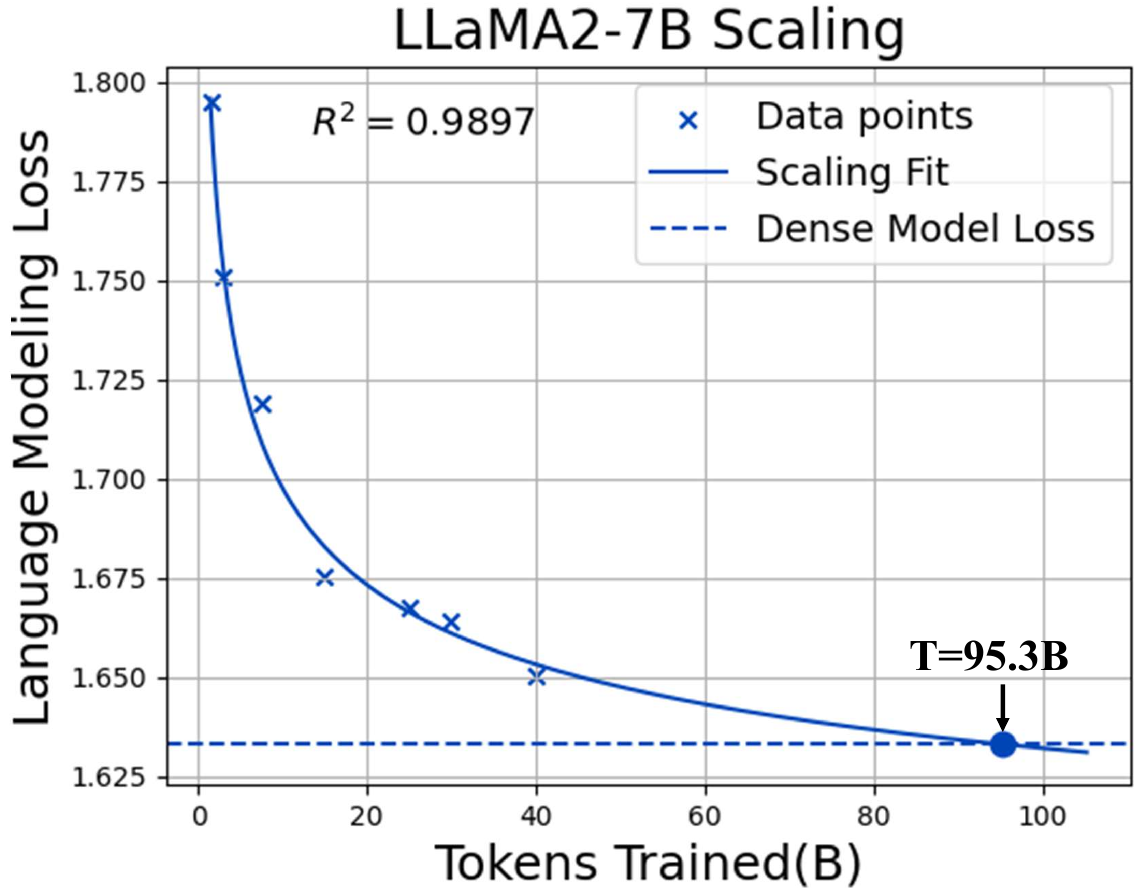}
        \caption{Scaling Tokens with LLaMA2-7B.}
    \end{subfigure}
        \hfill
    \begin{subfigure}[b]{0.32\textwidth}
        \includegraphics[width=\textwidth]{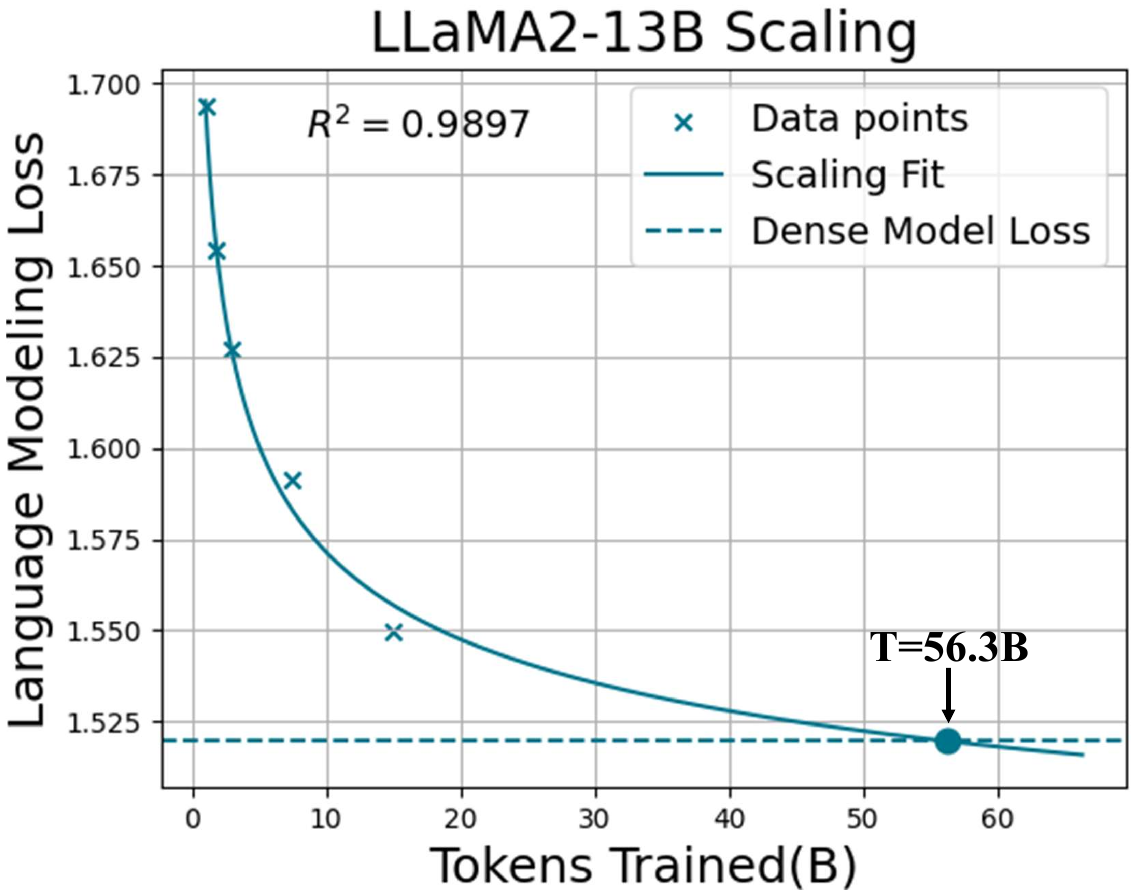}
        \caption{Scaling Tokens with LLaMA2-13B.}
    \end{subfigure}
    \hfill
    \begin{subfigure}[b]{0.32\textwidth}
        \includegraphics[width=\textwidth]{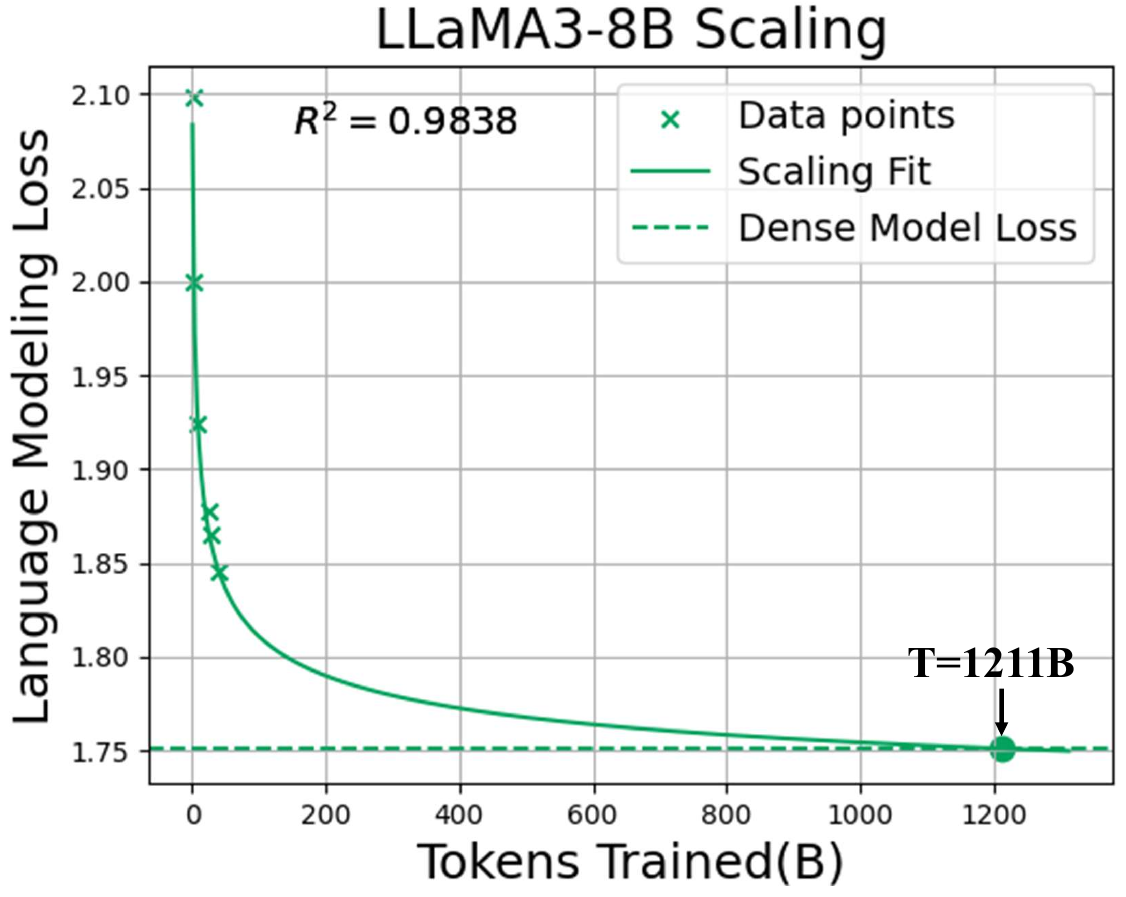}
        \caption{Scaling Tokens with LLaMA3-8B.}
    \end{subfigure}

    \caption{Scaling properties of LLaMA2-7B, LLaMA2-13B, and LLaMA3-8B.}
    \label{fig:6}
\end{figure*}

\subsection{Retraining Scaling Law}

\begin{table*}[t]
  \caption{Wikitext Perplexity Results using AWQ Quantization}
  \centering
\resizebox{0.9\textwidth}{!}{%
\begin{threeparttable}
\begin{tabular}{c|c|cccc|cc|ccc|ccc}
\toprule  
\multicolumn{2}{c}{Model} & 
  \multicolumn{4}{c}{GPT2} &
  \multicolumn{2}{c}{OPT} &
    \multicolumn{6}{c}{LLaMA} \\ \midrule

  \multicolumn{2}{c}{Metric} &
  \multicolumn{4}{c}{Perplexity} &
  \multicolumn{2}{c}{Perplexity} & 
    \multicolumn{3}{c}{Perplexity} & 
  \multicolumn{3}{c}{Zero-shot Accuracy(\%)}
 \\ \midrule
 
\makecell{Weight \\Pattern} &
Bits &
124M &
  350M &
  774M &
  1.5B &
  125M &
  1.3B  &
  2-7B &
2-13B &
3-8B &
2-7B &
2-13B &
3-8B \\ \midrule

\multicolumn{1}{c|}{\multirow{3}{*}{Dense}}  & 16 & 29.95 & 21.72 & 19.43 &17.40 & 27.76 & 14.62 & 5.12 & 4.57 & 5.76 & 57.16 & 59.60 & 59.74  \\
  &  4 & 33.78 & 22.52 & 19.73  & 17.70 &   29.08 & 14.93 &5.23  &  4.65 & 6.12  & 55.96  & 58.91  &  59.46  \\
  &  3 & 93.66 & 27.99 & 21.45  & 18.70 &   35.78 & 16.28 & 5.82  &  4.98 & 7.72  & 55.00  & 57.66  &  56.56  \\
  \midrule
\multicolumn{1}{c|}{\multirow{3}{*}{Sparse}}  & 16 & 31.24 & 23.12 & 19.97 & 17.35 & 28.32  & 14.92 & 5.21 & 4.71 & 6.33 & 57.52 & 58.93 & 59.43  \\
  &  4 & 35.75 & 24.48 &  20.21 & 17.49 &  29.31 &  15.07  & 5.29  &  4.84 & 6.42  & 57.22 & 58.91  &  59.11  \\
  &  3 & 75.83 & 29.24 &  22.19 & 18.44 &  33.60  &  15.82 &5.56  &  5.02 & 6.84 & 56.25  & 57.67  &  57.99  \\

  \bottomrule
\end{tabular} %
\begin{tablenotes}
\small
\item \textit{Note:} We use a group size of 128 for OPT and LLaMA models, and 64 for GPT-2 models due to compatibility constraints.
\end{tablenotes}
\end{threeparttable}
}
  \label{table:7}
\end{table*}

An important question we aim to address is: \textit{How many training tokens are required for a sparse model to fully match the performance of its dense counterpart?} Due to computational constraints, it is often infeasible to train models until full recovery is achieved. Therefore, we seek a practical method for estimating the training cost required to restore sparse model performance. This question is particularly relevant to CAST, as its dense-forward training strategy implies that the model only attains its final sparse form at the end of training.

A powerful tool to depict the scaling behavior is the Chinchilla scaling law \cite{10.5555/3600270.3602446}, which estimates language modeling loss $L(N, D)$ with:

\begin{align}
L(N, D) = E + \frac{A}{N^{\alpha}} + \frac{B}{D^{\beta}},
\label{eq:16}
\end{align}
where $N$ denotes the number of model parameters and $D$ the number of training tokens, $\alpha$ and $\beta$ are constant coefficients while $A,B$ vary across models. However, it is not feasible to establish the full Chinchilla scaling law in our case, as our experiments are limited to a small number of models pretrained on different datasets. Fortunately, by fixing the model architecture, we are able to isolate and examine the effect of the number of training tokens on model loss. This allows us to formulate a token-only scaling law independently for each model:

\begin{align}
L_i(D) = A_i + \frac{B_i}{D^{\beta}},
\label{eq:17}
\end{align}
where $A_i$ and $B_i$ are different parameters for each model. To evaluate model performance under different training budgets, we conduct a series of experiments with varying numbers of training tokens. For each setting, we tune the decay factor to ensure effective sparsity induction and optimal model performance. Furthermore, following the original Chinchilla paper \cite{10.5555/3600270.3602446}, we adopt $\beta=0.2849$ and fit our data points to find optimal $A_i$ and $B_i$ for each model using a linear regression model. We validate the suitability of this formulation in two ways. First, we compute the coefficient of determination ($R^2$) for each model fit and observe consistently high values (around 0.99) across all three models. Second, we estimate the performance of the longest-trained model by extrapolating from shorter training runs. The prediction error for perplexity remains below 0.03, further supporting the predictive validity of our approach.

\begin{table*}[t]
\centering
\caption{Speedup Results Using TensorRT‑LLM on RTX4090 and L20 GPUs With Different Input and Output Sequence Lengths, Measured by Throughput (Tokens/s). }
\resizebox{0.86\textwidth}{!}{%
\begin{tabular}{cc|ccc|ccc}
\toprule 
 \multicolumn{2}{c|}{Model} & \multicolumn{3}{c|}{LLaMA2-7B} & \multicolumn{3}{c}{LLaMA2-13B}  \\ \midrule
\multicolumn{1}{c}{GPU} & Inp Len, Out Len & Sparse & Dense & Efficiency & Sparse & Dense & Efficiency  \\ \midrule
\multicolumn{1}{c}{\multirow{4}{*}{L20}} & 128, 128 & 54.75 & 29.86 & 1.83x  & 35.42 & 17.81  &  2.00x  \\
\multicolumn{1}{c}{} & 128, 1024 & 53.81 & 29.57 & 1.82x  & 34.67 & 17.67  &  1.96x  \\
\multicolumn{1}{c}{} & 1024, 128 & 52.49 & 29.18 & 1.80x  & 33.88 & 17.51  &  1.93x  \\
\multicolumn{1}{c}{} & 1024, 1024 & 51.64 & 28.94 &  1.78x  & 33.51 & 17.40  & 1.92x   \\ \midrule
\multicolumn{1}{c}{\multirow{4}{*}{H800}} & 128, 128 & 108.54 & 75.54 &1.44x & 63.48 & 30.40  & 2.09x   \\
\multicolumn{1}{c}{} & 128, 1024 & 106.41 & 74.52 & 1.43x & 62.53 & 30.15 & 2.07x   \\
\multicolumn{1}{c}{} & 1024, 128 & 104.11 & 73.51 & 1.42x & 59.88 & 29.53  & 2.03x   \\
\multicolumn{1}{c}{} & 1024, 1024 & 101.05 & 72.45 & 1.40x & 58.80 & 29.34 &  2.00x \\ \midrule
 \multicolumn{2}{c|}{Memory Consumption} & 7.30GB  & 12.57GB & 0.58x & 13.95GB  & 24.25GB  & 0.57x  \\
\bottomrule
\end{tabular}
}
 \label{table:8}
\end{table*}

We present the scaling results in Figure~\ref{fig:6}, along with our predictions for the number of training tokens required to fully recover model performance. Notably, models with higher tokens-per-parameter ratios demand substantially more training data to match the performance of their dense counterparts. Specifically, the estimated token requirements for full recovery are 95.3B for LLaMA2-7B, 56.3B for LLaMA2-13B, and 1211B for LLaMA3-8B. This relationship appears approximately linear with respect to each model’s tokens-per-parameter ratio, echoing trends observed in the quantization scaling law literature \cite{kumar2024scalinglawsprecision}. In contrast, zero-shot accuracy does not exhibit the same trend, likely due to variations in the quality and domain coverage of the original pretraining datasets. Full data points and the corresponding fitted curves used for these predictions are provided in Appendix D.

\subsection{Combination with Quantization}

Model pruning can be further combined with quantization techniques to achieve greater compression ratios. To evaluate the compatibility of our semi-structured sparse models with existing quantization methods, we apply AWQ \cite{lin2023awq} to quantize our best-performing sparse models. Notably, in the original AWQ implementation, the zero-valued weights in our sparse model may be quantized to small non-zero values. To preserve the sparsity pattern even after quantization, we modify the quantization procedure to explicitly retain zero entries during quantization. Although empirically, we find that enforcing this constraint introduces no significant performance degradation.

Results are reported in Table~\ref{table:7}. Interestingly, we observe that the performance degradation for sparse models is significantly smaller than that of dense models at the same quantization rate. As a result, in certain low-precision settings, the sparse quantized model notably outperforms its dense counterpart. We hypothesize that this counterintuitive result stems from the fact that zeros in sparse models introduce no quantization error, effectively reducing the overall quantization noise and leading to less performance degradation. These findings further highlight the potential of combining weight sparsity and quantization for efficient model deployment.

\subsection{End-to-End Speedup and Compression Efficiency}

Table~\ref{table:8} presents the end-to-end decoding speedup results using TensorRT-LLM\footnote{\url{https://github.com/NVIDIA/TensorRT-LLM}}. We evaluate the 2:4 semi-structured sparse versions of LLaMA2-7B and LLaMA2-13B across two GPU architectures, using throughput (tokens per second) as the primary evaluation metric. Across varying input/output lengths and hardware configurations, the sparse models consistently achieve speedups ranging from 1.40× to 2.09× over their dense counterparts. In addition, the memory footprint for model weights is reduced to 57\% of the original size, further enhancing deployment efficiency. Notably, in some cases, the observed speedups exceed the theoretical 2× bound, which we attribute to reduced memory pressure allowing more efficient utilization of high-bandwidth caches.

\subsection{Finetuning Sparse Model on Downstream Tasks}

\begin{table}[t]
\centering
\caption{LoRA Finetuning Results on GSM8K. }
\resizebox{0.44\textwidth}{!}{%
\begin{tabular}{cccc}
\toprule 
Model  & LLaMA2-7B & LLaMA2-13B & LLaMA3-8B  \\ \midrule
Dense & 40.3\% &  49.4\%   &  69.3\%  \\
Sparse & 46.9\%  &  54.4\% & 65.4\%  \\
 \bottomrule
\end{tabular}
}
 \label{table:9}
\end{table}

Another important aspect we examine is the downstream performance of sparse models after fine-tuning. We evaluate both LoRA-based and full-parameter fine-tuning strategies. For LoRA fine-tuning, we assess LLaMA models on the GSM8K dataset using a rank of $r=64$ for both sparse and dense variants. As shown in Table~\ref{table:9}, the sparse models perform comparably to or even slightly better than their dense counterparts. For full parameter fine-tuning, we evaluate GPT-2 models on the GLUE benchmark using sparsity-aware fine-tuning, where only unmasked weights are updated. We compare this against traditional dense fine-tuning on dense models. As shown in Table~\ref{table:10}, sparsity-aware fine-tuning leads to only a minor performance drop relative to the dense counterpart, indicating that sparse models retain strong downstream adaptability despite reduced parameter count. Notably, allowing the sparse model to undergo dense fine-tuning yields no additional performance gains over sparsity-aware fine-tuning. Detailed results for each task are provided in Appendix E. Together, these findings suggest that semi-structured sparsity does not impair downstream transferability.

\begin{table}[t]
\centering
\caption{GLUE Finetuning Results on GPT2 Model}
\resizebox{0.43\textwidth}{!}{%
\begin{tabular}{ccccc}
\toprule   

 Weight Pattern  & 124M & 350M & 774M & 1.5B \\ \midrule
Dense     & 77.79 & 80.45 & 82.17 & 83.80    \\
 Sparse &  76.68 & 78.49 & 81.20 & 82.26 \\ 
  \bottomrule
\end{tabular} %
}

 \label{table:10}
\end{table}


\section{Conclusion}
In this paper, we introduce CAST (Continuous Adaptive Sparse Trainer), a fully continuous and differentiable sparsity-aware training framework designed to efficiently and accurately sparsify pretrained large language models into semi-structured formats. CAST addresses key limitations in prior work, particularly the lack of joint optimization between weights and masks, as well as inaccuracies imposed by sparse forward methods. To overcome these challenges, CAST introduces three core innovations: the AdamS optimizer, a learnable weight-scaling module, and knowledge distillation. Together, these components enable CAST to achieve state-of-the-art performance across multiple model families, producing lossless models under the 2:4 sparsity pattern. We further demonstrate CAST’s scalability by establishing empirical scaling laws and validating its robustness in downstream fine-tuning and quantization-based deployment. Looking ahead, we plan to extend CAST to structured pruning and explore its use in multi-modal and instruction-tuned LLMs, where efficiency and adaptability are critical.


%

\printbibliography

\ifCLASSOPTIONcaptionsoff
  \newpage
\fi



%



%
\begin{IEEEbiography}[{\includegraphics[width=1in,height=1.25in,clip,keepaspectratio]{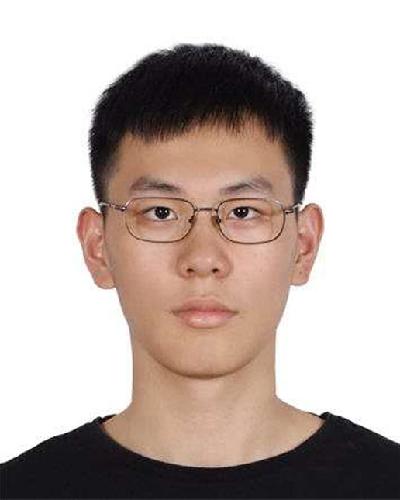}}]{Weiyu Huang}
received the BSc degree from the Department of Mathematical Sciences, Tsinghua University, China, in 2023. He is currently pursuing the PhD degree in the Department of Computer Science and Technology at Tsinghua University. His research interests include deep learning, efficient machine learning, and natural language processing.
\end{IEEEbiography}

\begin{IEEEbiography}[{\includegraphics[width=1in,height=1.25in,clip,keepaspectratio]{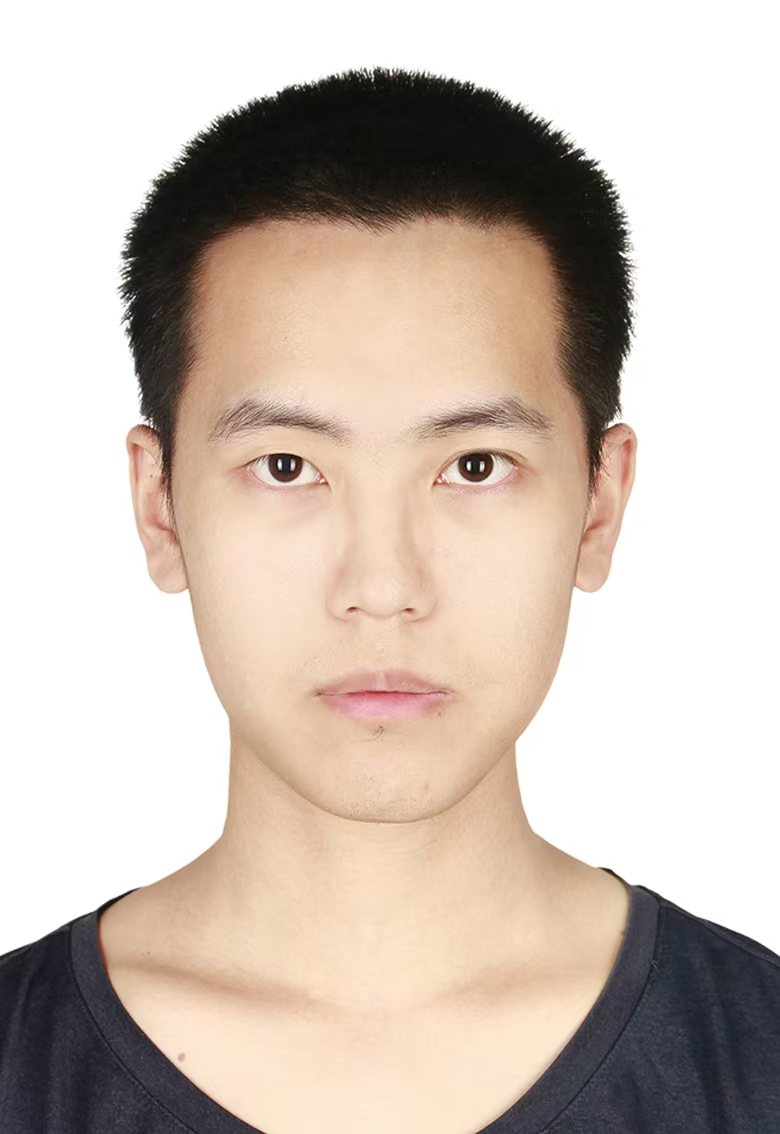}}]{Yuezhou Hu}
received the BEng degree from the Department of Computer Science and Technology, Tsinghua University, China, in 2025. His research interests include efficient algorithms for machine learning.
\end{IEEEbiography}

\begin{IEEEbiography}[{\includegraphics[width=1in,height=1.25in,clip,keepaspectratio]{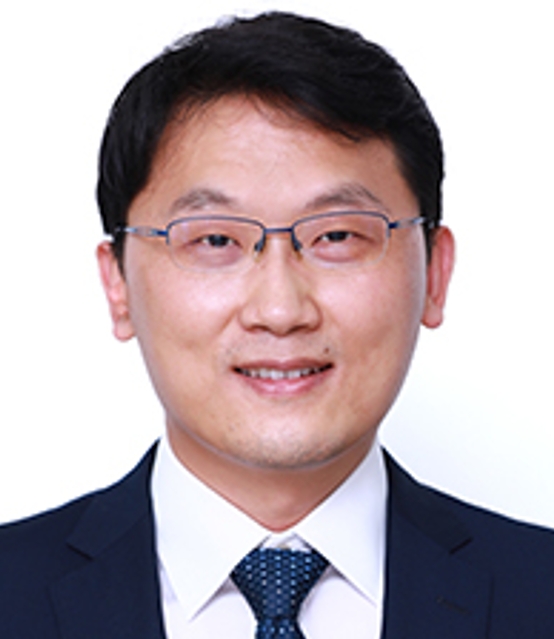}}]{Jun Zhu} (Fellow, IEEE)
received the BEng and PhD degrees from the Department of Computer Science and Technology, Tsinghua University, China, in 2005 and 2009, respectively. He is a tenured full professor in the Department of Computer Science and Technology at Tsinghua University. His research interests lie in machine learning and its applications in text and image analysis. He has published over 100 papers in prestigious conferences and journals. He serves as Associate Editor-in-Chief for IEEE Transactions on Pattern Analysis and Machine Intelligence (TPAMI). He has served as an area chair or senior program committee member for ICML, NeurIPS, IJCAI, UAI, AAAI, and AISTATS, and was the local co-chair of ICML 2014. He is a recipient of several honors, including the IEEE Intelligent Systems “AI’s 10 to Watch” Award, MIT TR35 China, the NSFC Excellent Young Scholar Award, and the CCF Young Scientist Award, and CCF first-class Natural Science Award. 
\end{IEEEbiography}

\begin{IEEEbiography}[{\includegraphics[width=1in,height=1.25in,clip,keepaspectratio]{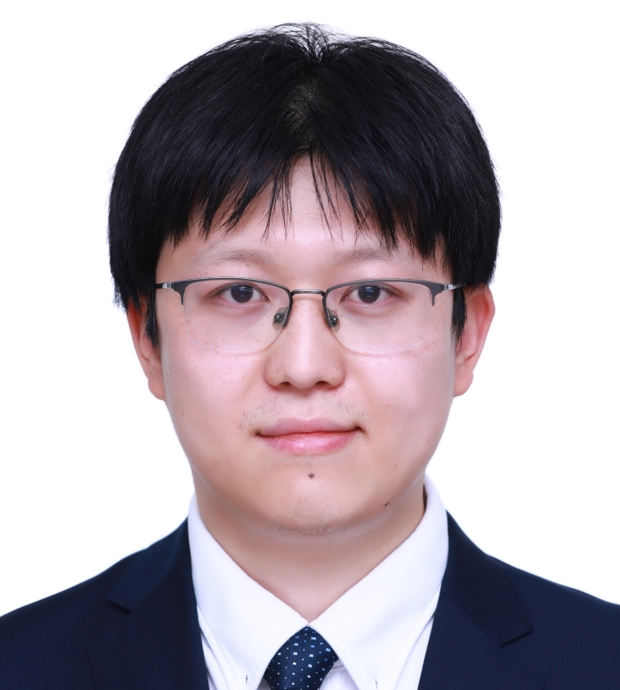}}]{Jianfei Chen} (Member, IEEE)
received the BEng and PhD degrees from the Department of Computer Science and Technology, Tsinghua University, China, in 2014 and 2019, respectively. He is currently an associate professor in the Department of Computer Science and Technology at Tsinghua University. His research interest is efficient machine learning, including low-precision training, sparse learning, mixture-of-experts. He is also interested in probabilistic inference and modeling. He served as an associate editor for IEEE TPAMI and an area chair in ICLR. (Corresponding author)
\end{IEEEbiography}








\appendices

\section{CAST Procedure} \label{append:cast}
In this section, we present the pseudocode for CAST in Algorithm \ref{alg:algorithm}. 

\begin{algorithm}[h]
\caption{Training Process for CAST}
\label{alg:algorithm}
\textbf{Input}: Total training iterations $T$; mask update frequency $T_1$; regularization strength $\lambda$; exponential decay rates $\beta_1$, $\beta_2$; weight Scaling hyperparameter $n$; Given current iteration $t$, we denote model weights as $\mathbf{\Theta}_t$; for each weight matrix $W_t^k\in \mathbf{\Theta}_t$, let $M_t^k \in \mathbf{M}_t$ denote its binary mask, $A_t^k$ be its weight scaling module;
\begin{algorithmic}[1] 
\FOR{$W^k_0 \in \mathbf{\Theta}_0$}
\STATE Get shape of $W^k_0$ as $(R_k,C_k)$;
\STATE Update weight mask $M^k_0$ according to Equation (6);
\STATE Initialize the weight scaling parameters $A_0^{k}  = \mathbf{1}^{R_k \times n}$;

\ENDFOR

\FOR{t = 1,2,..$T$}

\FOR{$W^k_{t} \in \mathbf{\Theta}_{t}$}
\IF{t$\mod$$T_1$ == 0}
\STATE Update weight mask {$M^k_{t}$} by Equation (6);
\ENDIF

\STATE Apply weight scaling through Equation (12);

\ENDFOR
\STATE Compute gradient through back-propagation on distillation loss $\mathcal{L}$ from Equation (13);

\FOR{$W^k_{t} \in \mathbf{\Theta}_t$}
\STATE Calculate $\alpha_{t} = \frac{t}{T}$;
\STATE Update parameters through AdamS in Equation (8).
\ENDFOR
\ENDFOR

\FOR{$W^k_{T} \in \mathbf{\Theta}_T$}
\STATE Conduct final pruning: $\hat{W}^k_{T} = W^k_{T} \odot M^k_{T}$;
\STATE Integrate weight scaling module into model weight through Equation (12);
\ENDFOR

\STATE \textbf{return} the sparse model.
\end{algorithmic}
\end{algorithm}

\section{Hyperparameters} \label{append:hyper}

\begin{table*}[t]
\centering
\caption{Summary of Hyperparameters and the Number of Tokens Used for Training}
\label{hyper}
\resizebox{0.9\textwidth}{!}{%
\begin{tabular}{c|ccc|cccc|ccc}
\toprule
& \multicolumn{3}{c|}{OPT} & \multicolumn{4}{c|}{GPT2}     & \multicolumn{3}{c}{LLaMA} \\ \midrule
Parameter Count& 125M    & 350M  & 1.3B   & 124M  & 350M  & 774M  & 1.5B  & 2-7B   & 2-13B  & 3-8B \\ \midrule
 
Learning Rate    & 1e-4  & 4e-5 & 2e-5 & 1e-4 & 1e-4& 4e-5& 4e-5& 2e-5& 2e-5& 2e-5 \\
Decay Coefficient  &  2e-6  & 2e-6 & 2e-6 & 1e-6 & 1e-6 & 1e-6& 1e-6& 4e-7& 4e-7& 4e-7  \\
Batch Size & 128  & 128 & 128 & 128 & 128& 128& 128&256& 256& 256  \\
Seqlen & 2048  & 2048 & 2048 & 1024 & 1024& 1024& 1024& 4096& 4096& 4096  \\
Training Steps & 30k & 30k & 30k & 60k & 60k& 60k& 60k& 7.5k& 7.5k& 7.5k  \\
Total Flipped Ratio  & 16.07\%  & 7.25\% & 5.59\% & 6.76\% & 5.95\% & 5.89\% & 6.91\% & 4.71\% & 3.40\% & 6.72\%\\

Kl Coefficient & 2/3  & 2/3 & 2/3 & 2/3  & 2/3 & 2/3 & 2/3 &1/3& 1/3 & 1/3\\
Tokens Trained & 7.5B  & 7.5B & 7.5B & 7.5B & 7.5B& 7.5B& 7.5B & 7.5B& 7.5B& 7.5B\\
\bottomrule
\end{tabular}%
}
\label{table:11}
\end{table*}

\begin{figure*}[h!]
    \centering
    \begin{subfigure}[b]{0.32\textwidth}
        \includegraphics[width=\textwidth]{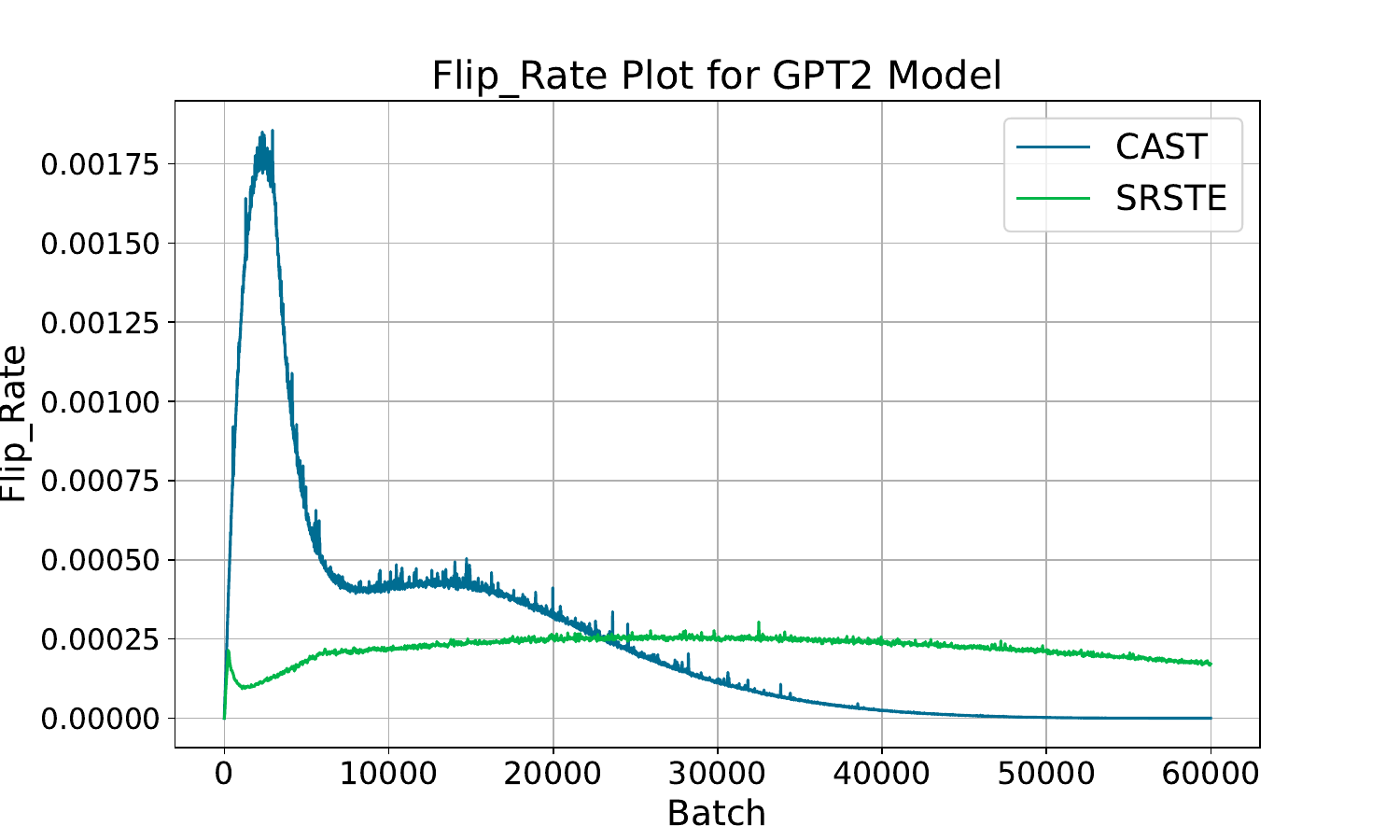}
        \caption{Flip Rate for CAST and SR-STE.}
    \end{subfigure}
        \hfill
    \begin{subfigure}[b]{0.32\textwidth}
        \includegraphics[width=\textwidth]{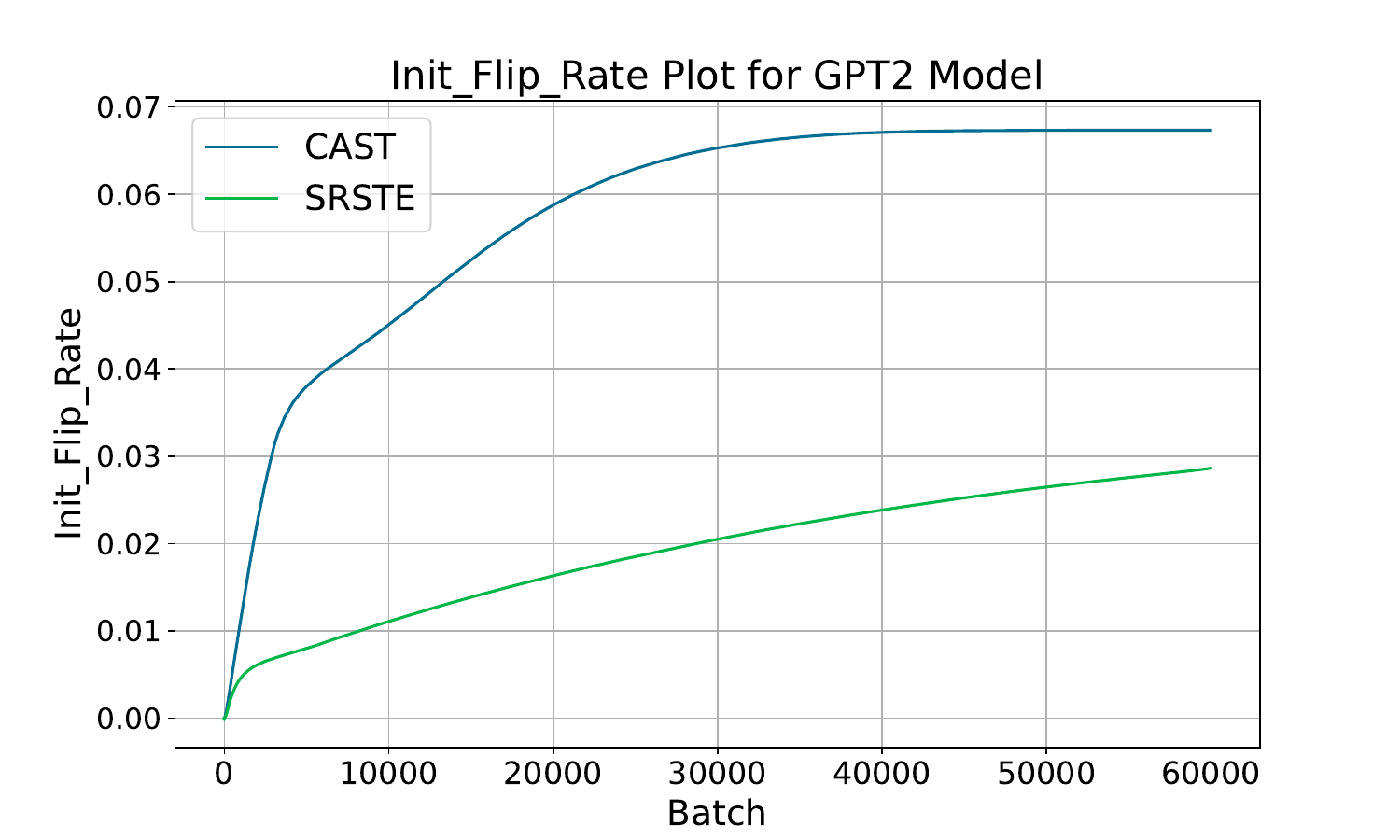}
        \caption{Init Flip Rate for CAST and SR-STE.}
    \end{subfigure}
    \hfill
    \begin{subfigure}[b]{0.32\textwidth}
        \includegraphics[width=\textwidth]{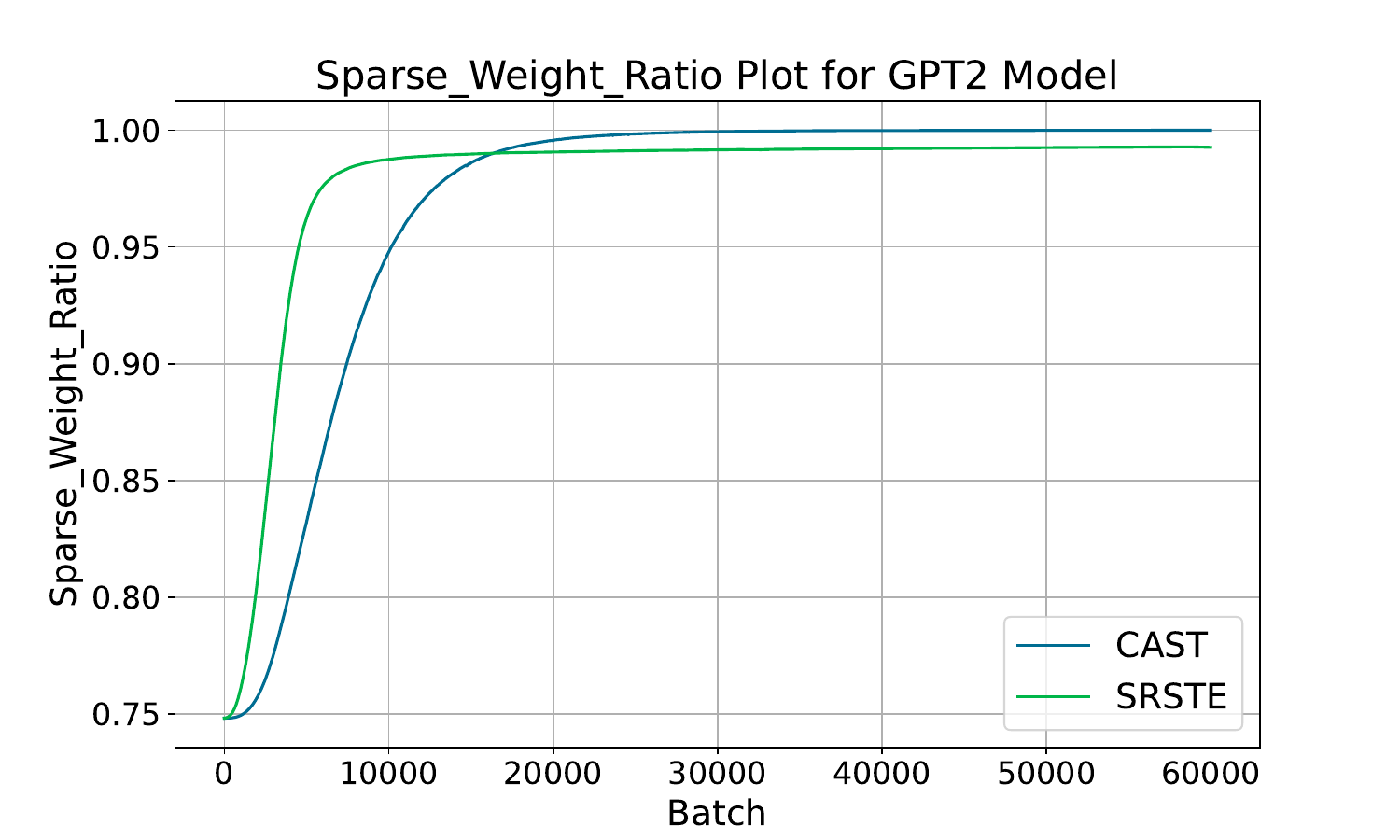}
        \caption{Spars. Wei. Rat. for CAST and SR-STE.}
    \end{subfigure}
    \caption{Flip rate, initial flip and sparse weight ratio change throughout the GPT2 model training process for CAST and SR-STE.}
    \label{fig:7}
\end{figure*}

This section summarizes the hyperparameters used during training. The decay factor is found to be robust and easy to tune, enabling the use of a consistent value across different model families. Learning rates follow the scaling guidelines from prior work \cite{kaplan2020scalinglawsneurallanguage}, with smaller values assigned to larger models. We use context lengths of 1024 for GPT-2, 2048 for OPT, and 4096 for all LLaMA models during both training and evaluation to ensure consistency across comparisons. Full hyperparameter configurations are provided in Table~\ref{table:11}.

\section{Mask Learning} \label{append:mask}

\begin{table*}[t]
  \caption{Detailed Model Performance in Scaling Experiments}
  \centering
\resizebox{0.9\textwidth}{!}{%
\begin{threeparttable}
\begin{tabular}{cccccccccccccc}
\toprule

Model  & 1B & 1.5B & 1.8B & 3B & 7.5B & 15B & 25B & 30B &  40B & $A_i$  & $B_i$ & $R^2$\\ \midrule

2-7B   & -  & 6.02 & - & 5.76 & 5.58 & 5.34 & 5.30 & 5.28  & 5.21 & 1.561 & 0.258 & 0.99\\ \midrule
2-13B   & 5.44 & - & 5.23 & 5.09 & 4.91 & 4.71 & - & - & - & 1.437 & 0.263 & 0.99\\  \midrule
3-8B    & - & 8.15 & - & 7.39 & 6.85 & - & 6.54 & 6.457 &  6.33 & 1.693 & 0.438 & 0.98\\
  \bottomrule
\end{tabular} %
\begin{tablenotes}
\small
\item \textit{Note:} For the LLaMA2-13B model experiments, due to computational constraints, we used a small number of training tokens for fitting. Cells marked with “–” indicate experiments that were not conducted due to computational constraints.
\end{tablenotes}
\end{threeparttable}
}
  \label{table:12}
\end{table*}

\begin{table*}[t]
  \caption{Finetuning Results on Downstream Tasks}
  \centering
\resizebox{0.9\textwidth}{!}{%
\begin{threeparttable}
\begin{tabular}{ccccccccccc}
\toprule  

 & &  \multicolumn{8}{c}{GLUE Task}  \\ \midrule
 
Model & \makecell{Weight \\Pattern} & \makecell{Finetune \\Pattern} &
 COLA & MNLI & MRPC & QNLI & QQP & RTE & SST-2  & Average \\ \midrule
\multicolumn{1}{c}{\multirow{3}{*}{GPT2}}  & Dense
    & Dense & 46.28 & 81.92/82.59 & 78.92/85.86&87.79& 89.48 & 64.26 & 92.09 & 77.79   \\
\multicolumn{1}{c}{\multirow{1}{*}{}} & Sparse & Dense & 39.97& 81.95/82.15& 78.92/85.95& 87.72& 89.20& 63.90& 91.51& 76.68   \\
\multicolumn{1}{c}{\multirow{1}{*}{}} & Sparse & Sparse & 37.31& 82.25/82.88& 79.17/86.27& 87.94& 89.66& 64.62& 91.97& 76.68  \\ \midrule
\multicolumn{1}{c}{\multirow{3}{*}{GPT2-Medium}}  & Dense & Dense &  51.80& 85.34/85.67& 81.13/87.10& 91.09& 90.91& 66.06& 93.69& 80.45 \\
\multicolumn{1}{c}{\multirow{1}{*}{}} & Sparse & Dense & 44.34& 85.12/85.47& 77.21/84.93& 90.17& 90.18& 62.45& 93.35& 78.12   \\
\multicolumn{1}{c}{\multirow{1}{*}{}} & Sparse & Sparse & 43.04& 85.15/85.49& 78.68/85.52& 90.37& 90.54& 64.26& 93.81& 78.49  \\ \midrule
\multicolumn{1}{c}{\multirow{3}{*}{GPT2-Large}}  & Dense    & Dense &  59.46& 86.30/86.35& 81.62/87.31& 91.82& 91.16& 67.87& 94.15& 82.17  \\
\multicolumn{1}{c}{\multirow{1}{*}{}} & Sparse & Dense &    57.58& 86.33/86.40& 79.90/86.42& 91.27& 91.25& 64.98& 94.27& 81.26\\
\multicolumn{1}{c}{\multirow{1}{*}{}} & Sparse & Sparse &  55.99& 86.21/86.34& 77.70/84.86& 91.38& 90.84& 68.59& 94.04& 81.20 \\ \midrule
\multicolumn{1}{c}{\multirow{3}{*}{GPT2-XL}}  & Dense  & Dense & 60.55& 87.00/87.21& 83.33/88.47& 92.18& 91.52& 74.37& 94.95& 83.80   \\
\multicolumn{1}{c}{\multirow{1}{*}{}} & Sparse & Dense &  58.88& 87.15/87.15& 85.05/89.57& 92.29& 91.48& 70.76& 94.61& 83.21  \\
\multicolumn{1}{c}{\multirow{1}{*}{}} & Sparse & Sparse &  56.76& 87.00/87.05& 82.35/87.80& 91.63& 91.49& 71.48& 95.18& 82.66 \\ 
  \bottomrule
\end{tabular} %
\end{threeparttable}
}
  \label{table:13}
\end{table*}

In this section, we describe several statistics that we monitored during the mask learning process. Specifically, at iteration $t$, we track the flip rate $r_t$, which measures the proportion of mask changes between consecutive steps, and the initial flip rate $i_t$, which quantifies the proportion of mask changes relative to the initial mask configuration. We also track the sparse weight ratio $S_t$, which represents the ratio of unmasked magnitude to the total magnitude. We give the definition as 
\begin{align}
     r_t = \frac{ \sum_{M^k_t \in \mathbf{M}_t}||M^k_t - M^k_{t-1}||_0}{\sum_{M^k_{t} \in \mathbf{M}_t} R_k  \cdot C_k} \nonumber
\end{align}
\begin{align}
          i_t = \frac{ \sum_{M^k_t \in \mathbf{M}_t} ||M^k_{t} - M^k_{0}||_0}{\sum_{M^k_{t} \in \mathbf{M}_t} R_k  \cdot C_k}  \nonumber
\end{align}
\begin{align}
     S_t= \frac{ \sum_{W^k_{t} \in \mathbf{\Theta}_t } ||W^k_t \odot M^k_{t}||_1}{ \sum_{W^k_{t} \in \mathbf{\Theta}_t } ||W^k_{t} ||_1}\nonumber
\end{align}
 The flip rate reflects the stability of the model's mask updates, while the initial flip rate measures the overall extent of mask learning. The sparse weight ratio quantifies the proportion of weights effectively zeroed out, indicating the potential performance impact of final pruning. We compare these metrics—flip rate, initial flip rate, and sparse weight ratio—between SR-STE and our proposed CAST. In our experiments, mask updates and flip statistics are computed every 10 batches, as more frequent updates provide no accuracy benefit and incur additional computational overhead.

As shown in Figure~\ref{fig:7}, CAST exhibits a higher initial flip rate than SR-STE, enabling more extensive exploration of sparsity patterns in early training. Toward the end of training, the mask stabilizes, promoting smoother convergence. This design allows CAST to support a higher rate of mask updates initially while maintaining overall training stability. Additionally, due to CAST's continuous weight decay mechanism, all masked weights converge to zero by the end of training, yielding a sparse weight ratio of 1. Thus, although CAST uses dense forward passes during training, the final pruning step introduces negligible performance loss. 

\section{Details on Scaling Law Experiments} \label{append:scaling}

In this section, we present the details of our scaling law experiments. To ensure effective mask learning, we maintain a consistent learning rate for each model and adjust the decay factor based on the training token budget. The data points used to fit the linear models are provided in Table~\ref{table:12}. We observe that the coefficient of determination ($R^2$) for the fitted models is approximately 99\%, indicating a strong fit and validating the suitability of the Chinchilla-style scaling approach.

To further assess the predictive capability of our model, we conduct a leave-one-out analysis by excluding the final data point and using the remaining data to estimate its value. For LLaMA2-7B and LLaMA3-8B, the predicted perplexities are 5.23 and 6.34, with absolute errors of just 0.02 and 0.01, respectively. In the case of LLaMA2-13B, the prediction error is higher, likely due to a smaller number of available data points for fitting. Interestingly the value of $B_i$ for both LLaMA2 model are similar providing more evidence that the sparse retraining scaling law conforms with Chinchilla's claim. Overall, these results demonstrate the robustness and effectiveness of our scaling model in estimating performance under varying training budgets.

\section{Detailed Results on GLUE Finetuning Tasks} \label{append:finetune}

We present the detailed results of GLUE fine-tuning in Table~\ref{table:13}. For sparse models, we evaluate two settings: (1) standard fine-tuning, where all weights—including previously masked ones—are updated, effectively converting the model back to a dense form; and (2) sparse-aware fine-tuning, where masked weights remain fixed at zero, and only unmasked weights are updated. Notably, we did not observe any performance improvement from prior methods that introduce additional parameters during fine-tuning, despite their increased complexity.

\section{Time Complexity Comparison} \label{append:time}

In this section, we analyze the time complexity of CAST. We measure the training time per mini-batch on an 8×H800 GPU setup for GPT and OPT models and a 32×H800 GPU setup for LLaMA3-8B. Theoretically, knowledge distillation methods should require approximately $\frac{4}{3}$ times the FLOPs compared to non-distillation methods, translating into proportionally increased training time under purely compute-bound conditions. However, our empirical results (see Table~\ref{table:14}) show lower than theoretical overhead for smaller models, indicating compute is not the bottleneck. LLaMA2-7B experiences greater overhead than predicted by theory, suggesting it is predominantly compute-bound. Moreover, MaskLLM was trained for 1200 A100 GPU hours on LLaMA2-7B, whereas our method and other baselines are trained using roughly 600 H800 GPU hours.
\begin{table}[t]
\centering
\caption{Comparison for Time Complexity}
\resizebox{0.47\textwidth}{!}{%
\begin{tabular}{ccccc}
\toprule   

 Model  & \makecell{Naive \\ Retraining} & SR-STE & \makecell{CAST w/o \\ KD} & CAST \\ \midrule
GPT2     & 193.56ms & 200.37ms & 202.32ms & 278.24ms    \\
OPT-125M &  195.30ms & 203.80ms & 205.20ms & 284.99ms \\ 
LLaMA3-8B  &  261s & 276s & 281s & 403s \\ 
  \bottomrule
\end{tabular} %
}

 \label{table:14}
\end{table}

\section{SR-STE Gradient Analysis} \label{append:srste}
SR-STE maintains a sparse forward process and updates masked weights using the STE approximation. Furthermore, it applies $L_2$ decay to mitigate mask oscillation. Specifically, for a parameter $\theta_t$ at the $t^\text{th}$ iteration, the weight is first multiplied by a binary mask, resulting in $\hat{\theta}_t = \theta_t \cdot m_t$, which is then used in the forward pass. During the backward pass, let $g(\cdot) = \frac{\partial \mathcal{L}}{\partial \theta}(\cdot)$ denote the gradient of the loss with respect to $\theta$. Under this notation, the approximated element-wise update rule for the parameter is given by:

\begin{equation}
\theta_{t+1} =
\theta_{t} - \gamma_t \left(g(\hat\theta_{t}) + \lambda \theta_{t} \right) \nonumber
\end{equation}

where $m_t$ and $\gamma_t$  are values of weight mask and learning rate at iteration $t$ respectively. To illustrate the effect of inaccurate gradient estimator, we calculate the  gradient error of this approximation when updating masked weights, as the difference between the estimated and true gradients can be expressed as:

\begin{align}  
   \Delta_t &= \left| g(\theta_t) - g(\hat{\theta}_t) - \lambda \theta_t \right| \notag \\
       &= \left| g'(0) (\theta_t + o(\theta_t)) - \lambda_t \theta_t \right| \notag \\
          &\approx \left| g'(0) - \lambda \right| \cdot \left| \theta_t \right| \propto \left| \theta_t \right| \nonumber
\end{align}

As shown above, under certain approximations, gradient estimation becomes increasingly inaccurate for weights with larger magnitudes.

\section{Distillation Function Ablation}
In this section, we provide the results of ablation studies 
on different distillation functions which use attention and hidden states for distillation, including TinyBERT , MobileBERT, Sparse-Finetuning, as well as MiniLLM, which employs reverse KL divergence. As shown in \ref{table:15}, we find that using intermediate information during the distillation of generative language models is detrimental, and that KL loss is sufficient for optimal performance. 

\begin{table}[t]
\centering
\caption{Ablation study on different distillation loss for training sparse models.}
\resizebox{0.45\textwidth}{!}{%
\begin{tabular}{ccccc}
\toprule
 & \multicolumn{3}{c}{GPT} & OPT \\ \cmidrule(lr){2-5} 
 Method    & 124M & 350M& 774M& 125M \\ \midrule
Dense   & 29.95 & 21.72   &19.43  &27.76\\  \midrule
TinyBERT         & 40.38  & 32.20  & 28.62 & 37.14\\ 
MobileBERT       & 42.60 &   30.92  &  28.08 & 38.46  \\ 
Sparse Fintuning & 39.57&   28.73   & 26.54& 36.62  \\ 
MiniLLM          & 31.27 &   23.14   & 19.96 & 28.35   \\ 
KL Loss(Ours)  & 31.24 &   23.12 & 19.97 & 28.32\\
\bottomrule
\end{tabular}
}
\label{table:15}
\end{table}

\end{document}